\documentclass[11pt]{article}
\usepackage[margin=1in]{geometry}
\usepackage{amsmath,amssymb,amsthm}
\usepackage{booktabs}
\usepackage{pifont}          % \ding{51}/\ding{55}:related-work 定位表的对勾/叉
\usepackage{graphicx}
\usepackage{hyperref}
\makeatletter
\g@addto@macro\ttfamily{\hyphenchar\font=`\-\relax}
\makeatother
\let\WCorigunderscore\_
\renewcommand{\_}{\WCorigunderscore\allowbreak}
\newtheorem{theorem}{Theorem}

\title{Pricing the Risk of Runtime Compression:
Anytime-Valid Admission and a Served-Output Law for
  Compressed Serving State}
\author{%
  Fanzhe Wei\\ Metask Lab\\ {\small\texttt{whyer1@gmail.com}}
  \and Li Liu\\ Metask Lab\\ {\small\texttt{muriel092611@gmail.com}}}
\date{}

\begin{document}
\maketitle

\begin{abstract}
Runtime compression of serving state trades quality for capacity with
no priced guarantee: systems adapt precision on load signals with no
soundness statement, and certified approaches budget request-level risk
by a union bound over a pre-declared event count. We show the union
budget exhausts on every long request in a production serving stack
($100\%$ of requests), and replace it with an anytime-valid,
physically-accounted admission ledger whose bound holds at every one of
$352{,}333$ admission calls on live traffic and which, in a
pre-registered held-out confirmatory round, halves the exact-fallback
rate at matched risk ($0.30\to0.14$) --- coverage is bought at a price
the account states. We then price the remaining distance from the
certified witness to what a user experiences: a machine-checked design
law ($\mathrm{TV}\le\tanh(a_q w_{\mathrm{thr}})$) turns the served-TV
target into a threshold knob, and a three-layer audit of its
instantiation --- an operator-norm query envelope measured $1.5\times$
from tight, a measured-ellipsoid replacement for the Cauchy--Schwarz
ball that buys nothing ($0.89\times$, held-out sound), and the gate's
operating point ($\sim\!700\times$) --- localizes the entire
$1064\times$ gap to the operating point, a price the law now states
rather than an unknown. A priced bound is still worth nothing on a
request one has not seen, so the third link is the quantifier:
exchangeable extrapolation across $80$ serving histories replaces
binary conformal prediction's vacuous certificates with order-statistic
bounds that discriminate ($0.41$ against $0.51$ calibration risk). All probabilistic kernels are Lean~4-checked ($228$
exported theorems, no \texttt{sorry}); which object deserves this
machinery at all is settled empirically in a companion paper
(\emph{What to Protect When You Quantize a Mixture of Experts}) that
adjudicates --- and rejects --- the natural alternative of certifying
expert routing. What ships is an account: risk you can spend, a gap you
can read off a law, and a bound that survives the request you have
not seen.
\end{abstract}

\begin{figure}[htbp]\centering
\includegraphics[width=\linewidth]{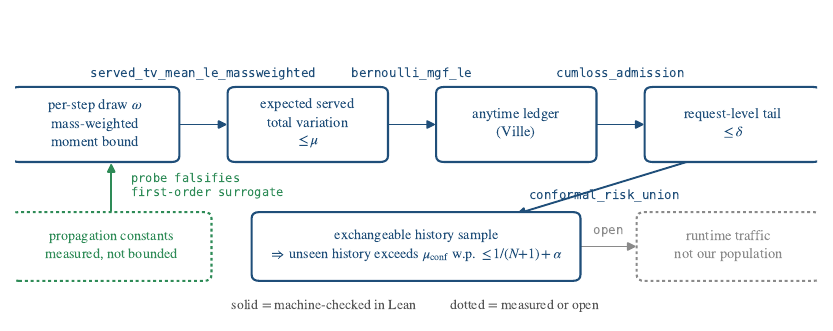}
\caption{The chain this paper machine-checks, and where it stops. Solid
boxes and their connecting theorems are proved in Lean with no
\texttt{sorry}; dotted boxes are measured or open. Reading left to right:
a per-step moment bound on the rounding draw becomes a bound on expected
served-output total variation, Bernoulli domination feeds it to the
anytime ledger, and Ville's inequality closes the request. The lower row
is the quantifier: exchangeable history sampling lifts a fixed-history
bound to an unseen history, and the one box we cannot fill is that our
population is a pool we constructed rather than runtime traffic.
The propagation constants that a tensorized proof would need are measured
and \emph{falsified}, not bounded --- which is why the top row is stated
for an arbitrary step perturbation.}
\label{fig:chain}
\end{figure}

\section{Introduction}
Quantization for large-model serving is decided at the wrong
timescale. Post-training quantization chooses a \emph{static}
configuration offline and defends it with average-case evidence;
serving systems increasingly change precision \emph{at runtime} (layer
swapping under memory pressure, learned bit-plane routing, phase-split
precision) with no soundness statement. The structural gap: no system
decides precision from the \emph{content of the current request} while
providing a sound, per-request accounting of the risk that decision
incurs.

For MoE models the gap is sharper: quantization error also flips
discrete routing decisions (on Mixtral-8$\times$7B at extreme
compression, expert selection changes on over $40\%$ of tokens ---
GEMQ~\cite{gemq2026}), and the published response is entirely
\emph{offline}. Our own
measurement sharpens the point: quantizing the router itself to INT4
flips $63.2\%$ of top-6 selections on DeepSeek-V2-Lite, because the
quantity a routing certificate would protect is an order of magnitude
smaller than the noise it would have to survive:
\begin{equation}
\underbrace{\operatorname{med}\bigl(s_{(k)}-s_{(k+1)}\bigr)}
      _{\text{routing margin}}=0.023
\;\ll\;
\underbrace{2L\,\operatorname{med}(\varepsilon_\infty)}
      _{\text{gate noise}}=0.354
\quad (L=1) .
\label{eq:margin-vs-noise}
\end{equation}
This explains,
with numbers, a practice the industry adopted by instinct (production
MXFP4 checkpoints exclude \texttt{mlp.gate} without exception) --- and
it identifies the quantity an \emph{online} certificate must protect:
the margin between the $k$-th and $(k{+}1)$-th routing logits against
\emph{upstream} precision noise.

This paper prices runtime compression of serving state. The account is
one chain --- an object, its price, and the quantifier that decides
whether the price means anything on the request not yet seen --- and we
state for each link both what it buys and where it stops:
\begin{itemize}
\item \textbf{The object: a risk that can actually be spent.} Prior
  certified systems budget per-event failure by union bound and state
  that no union bound survives free-running decoding. The tempting fix
  --- let the e-process itself admit actions --- changes the
  guarantee's object (it controls wrongly rejecting an honest radius
  model, not admitted-action failure; we give the exact
  counterexample). We close the sound branch instead: a machine-checked
  anytime-valid budget on \emph{cumulative realized loss}
  (\texttt{cumloss\_admission}) whose per-event terms the write-path
  audit already measures (\S\ref{sec:av}).
\item \textbf{The price: what one unit of that risk buys a served
  token.} A machine-checked chain carries the witness's per-layer
  variance budget to a request-level served-TV tail bound
  (\texttt{request\_tail\_of\_served\_tv}), and a design law turns the
  served-TV target into a threshold knob; a three-layer audit of that
  law's instantiation localizes the entire $1064\times$ gap between the
  certified witness and the served output to the gate's operating point
  --- a price the law now states rather than an unknown
  (\S\ref{sec:instantiate}, \S\ref{sec:perread}). Open and stated as such: the model-side
  propagation constants (measured, and the measurement \emph{falsifies}
  the first-order surrogate) and the live controller beyond its
  $+3.8\%$ judgment cost.
\item \textbf{The quantifier: whether any of it holds on the next
  request.} A bound calibrated on the histories one happens to have is
  not yet a bound on the history about to be served. Exchangeable
  without-replacement sampling over $80$ serving histories lifts it and
  replaces binary conformal prediction's vacuous certificates with
  order-statistic bounds that discriminate ($0.41$ against $0.51$
  calibration risk, \S\ref{sec:quantifier}). The premise we cannot supply is that
  our population is a pool we constructed rather than runtime traffic;
  we say so wherever the number is used.
\end{itemize}
Two things this paper deliberately does not do. It does not certify
expert routing: that alternative object we built first and rejected on
measured grounds, an adjudication that is the companion measurement
paper's subject\footnote{Companion measurement paper, \emph{What to Protect When You Quantize a Mixture of Experts}, by the same authors, posted concurrently. No result of that paper is claimed here, and this paper is self-contained without it.} and none of whose results are claimed here ---
\S\ref{sec:router} states only the boundary it draws, which is why the
object priced above is the served output rather than the routing
decision. And it does not bury the kill-shots that shaped the design:
the gate-quantization arm above, and the pre-registered death of a
value-transport rescue hypothesis for KV quantization (synthetic pool:
rank correlation $0.872$, zero rescuable condemned reads; real-pool
recheck: the production quantizer family never reaches the condemnation
stratum, adjudicated ``inconclusive, no high shift'', metric-to-error
correlation $0.935$), which redirected this paper from KV distortion
metrics to precision \emph{decisions}.

Separately, the precondition study for request-conditioned geometry ---
which no bound in this paper consumes --- is reported in
Appendix~\ref{app:geometry} rather than claimed as a result.

\begin{table}[htbp]\centering\small
\setlength{\tabcolsep}{4pt}
\resizebox{\linewidth}{!}{%
\begin{tabular}{@{}llrrl@{}}
\toprule
Model & Routing & Experts / top-$k$ & Layers & Where it carries the paper\\
\midrule
DeepSeek-V4-Flash-FP8 & sqrtsoftplus, bias-corrected & $256$ / $6$ & $43$
  & served-output bridge; anytime ledger; physical gate;\\
 & & & & \quad $80$-history conformal extrapolation\\
GLM-5.2 (FP8, W4A-FP8) & sigmoid, bias-corrected & $256$ / $8$ & $78$
  & capacity--cost evaluation in the production stack\\
DeepSeek-V2-Lite & softmax, greedy & $64$ / $6$ & $27$
  & router-margin certificate; flip$\to$damage arms\\
Qwen1.5-MoE-A2.7B & softmax, greedy & $60$ / $4$ & $24$
  & cross-model replication of the flip arms\\
Qwen2.5-7B & dense & --- & $28$
  & activation-geometry precondition (dense control)\\
\bottomrule
\end{tabular}}
\caption{Model coverage. The served-output line runs on a
production-scale fine-grained MoE ($256$ experts, top-$6$); the
capacity evaluation on a second, architecturally distinct one; the
routing-certificate arms on smaller greedy-top-$k$ MoEs, treated as a
same-family proxy (\S\ref{sec:router}).}
\label{tab:models}
\end{table}

\paragraph{Scope of the controlled action, stated up front.}
Table~\ref{tab:models} lists which model carries which result. The
admission engine is object-agnostic --- its inputs are audited
per-event losses and predictable MGF terms, whichever component
produces them --- but the action it controls in this paper is
\emph{KV-write precision} (dithered compression vs.\ exact retention)
inside an MoE serving stack. Expert-weight precision actions --- the switch whose economics
\S\ref{sec:exp} measures at $3.73\times$ capacity headroom --- are the
engine's designated next object, not a delivered one. We price that
headroom rather than quote it: pressure-probed at $26$k-token requests
it converts to $2.0$--$2.24\times$ the in-flight requests, and from
there to \emph{nothing} when prefixes never repeat and to
$1.19\times$ throughput (a measured lower bound) when each session
reuses its own context (Table~\ref{tab:convert}); which it becomes is
decided by the workload, and we report both. The switch itself --- per-component
weight-plane management --- is motivated here, not implemented: an
engine proved on the first object, the second object's value
quantified and its mechanism left open.

Everything is measured under a provenance-stamped, machine-adjudicated
protocol: every gate in this paper was pre-registered with its verdict
branches written before the data was seen, and every number traces to a
run artifact. Table~\ref{tab:adjudication} is the full ledger of those
branches and the verdict each one fired.

\section{Setup and inherited contracts}
We inherit the typed-contract vocabulary of our prior systems
(WitCert~\cite{witcert2026} and its observability
companion~\cite{witprobe2026}): typed
input/output metrics with mismatch
rejection; certified / partial / empirical tiers; per-request ledgers;
and machine-checked (Lean~4) soundness for the probabilistic kernel.
Notation: a serving process emits memory/compute events $t=1,2,\dots$
with filtration $(\mathcal F_t)$; a precision plan assigns each linear
component $\ell$ an action $a_\ell$ (bit-width, exact-keep, or
fallback); request-level risk is budgeted as $\delta_{\mathrm{req}}$.
The whole paper turns on one structural requirement, which we state once
and then never relax:
\begin{equation}\label{eq:predictable}
  a_t\in\mathcal F_{t-1},
  \qquad
  L_t\in[0,1]\ \text{is}\ \mathcal F_t\text{-measurable},
  \qquad
  \mathbb{E}\big[L_t\mid\mathcal F_{t-1}\big]\le\mu .
\end{equation}
Decisions must be \emph{predictable} --- chosen before the loss they
incur is revealed --- because an anytime bound over a stream that adapts
to its own outcomes is otherwise not a bound at all.

\subsection{What is inherited and what is new}
Sound runtime certificates for KV-cache quantization exist (ours and
independent work); request conditioning, sound per-read bounds, and
serving-loop evidence are established there for the \emph{KV} domain,
and we do not re-claim them. This paper's claims live strictly in the
two vacant axes: anytime-valid budgets that \emph{drive} decisions,
and the online MoE routing question --- the latter adjudicated
negatively in the companion paper (\S\ref{sec:router}).

An ellipsoid radius conditioned on the request would need activations to
concentrate in low-dimensional subspaces; we test that precondition with
held-out controls and report it in Appendix~\ref{app:geometry}, since no
bound in this paper consumes it.

\section{What not to certify: the routing adjudication, in brief}
\label{sec:router}

A natural first candidate for the certified object is expert-routing
invariance, and we built that certificate before pricing it: a sound
selection-domain margin test with per-model Lipschitz constants,
adjudicated across nine MoE families ($485{,}138$ tokens, zero
soundness violations). The verdict, reported in full in the companion
measurement paper, is negative on three independent grounds --- the
gate binds on $94$--$100\%$ of production traffic and remains unusable
after every sound tightening we could construct (including a measured
per-coordinate ellipsoid); the flips it prevents anti-predict damage;
and set identity provably fails to bound the routing distribution it
was meant to protect. Replaying full-precision routing recovers a
constant $8$--$17\%$ of quantization damage across every knob the
companion paper turns, worth $0.75\%$ of base perplexity at the
four-bit depth production ships, and becomes actively harmful past a
model-dependent damage cliff. We therefore take the routing question as
\emph{settled against certification} and spend this paper on the object
the adjudication points to: the served output's cumulative tail under a
request-level risk budget, which the rest of this paper defines, prices,
and machine-checks.

\section{Anytime-valid admission}
\label{sec:av}
Certified systems to date spend request-level risk by union bound
over a \emph{pre-declared} event count:
\begin{equation}
\delta_t \;=\; \frac{\delta_{\mathrm{req}}}{(t+1)(t+2)},
\qquad
\sum_{t\ge 0}\delta_t \;=\; \delta_{\mathrm{req}} ,
\label{eq:unionbudget}
\end{equation}
which is sound but pays for horizon in advance. Our own prior system
states the limitation precisely: free-running decoding feeds quantized
cache back into future queries, an adaptive dependence ``no union bound
repairs''; its e-process was deployed only as \emph{model validation},
``complementing, not replacing'' the union budget that authorization
consumes.

The tempting construction --- admission gated by the e-process wealth
itself --- is the one we implemented first, and it is \emph{wrong for
the stated guarantee}. Ville's inequality bounds the probability that an
honest radius model ever looks anomalous; it says nothing about the
probability that an admitted action fails. The two differ by more than a
constant. A $2{,}000$-event stream admitted at per-event risk $10^{-4}$
passes the wealth test while incurring an $18\%$ any-failure
probability, so wealth-gating answers a question nobody asked. The
corrected object is a budget on \emph{cumulative realized loss}, and its
admission rule is machine-checked (\texttt{cumloss\_admission}): if the
account is maintained at rate $\lambda$ with the per-event terms the
write-path audit already produces, then
$\Pr(\exists t\le T:\sum_{s\le t}L_s>B_t)\le\delta$. Two facts make this
implementable rather than aspirational. The gate is conservative in the
right direction --- admitting on batch-level upper bounds while settling
on the subset that actually passes preserves the budget
(\texttt{gating\_conservative}) --- and at the request endpoint the certified
account is tight against the physical witness it charges rather than
orders of magnitude loose (median margin of bound over witness $4.6\%$
across $16$ endpoints, witness below bound on all of them). That is an
\emph{endpoint} statement about accumulated totals: the run does not
record a per-event margin trajectory, so individual events may run
looser and offset.

\begin{table}[htbp]\centering\small
\begin{tabular}{@{}lrrrr@{}}
\toprule
& union & \multicolumn{3}{c}{cumloss, by budget $\tilde B$}\\
\cmidrule(l){3-5}
& & $5\!\times\!10^{3}$ & $2\!\times\!10^{4}$ & $5\!\times\!10^{4}$\\
\midrule
budget exhaustion & $100\%$ & --- & --- & ---\\
fallback-to-exact rate (median) & $0.315$ & $0.842$ & $0.339$ & $0.0079$\\
admission calls (armed) & $0$ & $641{,}284$ & $1{,}336{,}588$ & ---\\
judgment cost (s/request) & $344.5$ & $339.5$ & $357.1$ & ---\\
violations & $0$ & $0$ & $0$ & $0$\\
\bottomrule
\end{tabular}
\caption{Dual-accounted admission, measured in a production serving
stack. The union budget exhausts on every long request; cumulative-loss
admission does not --- what it buys is a \emph{monotone} coverage/budget
curve, not a free lunch, and each column must be read at its own budget:
at $\tilde B=2\!\times\!10^{4}$ the fallback rate ($0.339$) is
\emph{worse} than the union baseline's $0.315$, and the pre-registered
mechanism control recorded exactly that, so that arm's adjudication is
reported \texttt{INVALID} rather than quoted as a win. The
$31.49\%\!\to\!0.79\%$ drop is real but belongs to the $2.5\times$
wider budget $\tilde B=5\!\times\!10^{4}$. Zero authorization
violations across all arms. Full adjudication notes
(cross-run calibers, non-monotone judgment cost, physical-gate
non-vacuity) in Appendix~\ref{app:ledgernotes}.
}
\label{tab:ledger}
\end{table}

\begin{figure}[htbp]\centering
\includegraphics[width=0.58\linewidth]{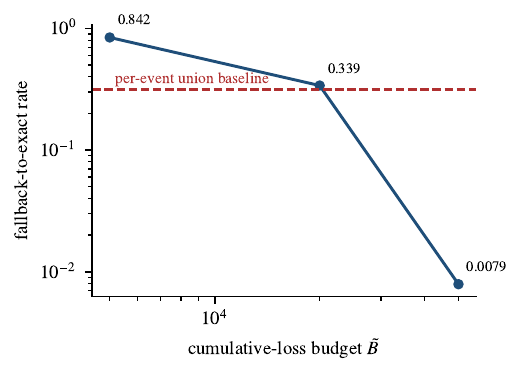}
\caption{Coverage is bought, not found. Fallback-to-exact rate vs.\
cumulative-loss budget (\emph{exploratory} normalized-$z$ arm), union
baseline for reference: at Table~\ref{tab:ledger}'s budget the
cumulative account is \emph{worse} than union; the large gain belongs
to a $2.5\times$ wider budget.}
\label{fig:coverage}
\end{figure}

We then close the loop: a physical-gate armed run
(\texttt{WITCERT\_CUMLOSS\_MODE=phys}, frozen commit, nine
pre-registered gates all passing) drives KV-write admission by the
\emph{physical} account itself --- $317{,}293$ and $352{,}333$
admission calls in the loose and tight arms. Three properties hold on
live data. The physical anytime bound is never crossed: per account, at
every admission call,
\begin{equation}
\max_{t}\;\bigl(\text{realized charge}_t-\text{admitted charge}_t\bigr)
\;\le\;0 ,
\label{eq:physgap}
\end{equation}
the $L{=}W$ instantiation of \texttt{cumloss\_admission} realized, and
the runtime counterpart of \texttt{gating\_conservative} and
\texttt{admit\_implies\_realized\_within} --- admitting on the
full-batch upper bound is provably safe per step, here measured
account-by-account. The controller modulates by budget (tight
exact-fallback $0.48$ vs.\ loose $0.22$ against the same round's union
baseline $0.31$). And in a pre-registered held-out confirmatory round
(frozen $\lambda$ and budget, topic-disjoint), all three properties
replicate and the coverage gain becomes a \emph{same-round} comparison:
physical arm $0.14$ vs.\ its own union $0.30$. Not recorded: the margin
trajectory $\max_t(\mathrm{realized}_t-\mathrm{prospective}_t)$ --- the
zero-violation fact says only that the conservatism's slack is
non-negative. The judgment costs $+3.8\%$ wall time.

Two honest caveats remain. $\lambda$ and the budgets were calibrated
in one exploratory round and confirmed in one frozen held-out round ---
a single replication, not a sweep; and $W$ is a
local certified witness (band-norm of K/V residual), \emph{not}
served TV/NLL. The served-output leg has a three-part status.
\emph{Per read} it is sound: the inherited machine-checked
\texttt{tv\_le\_eform} bounds one attention softmax's TV by
$\tfrac12(A_c^2{-}1)$. Soundness is not usefulness
(Figure~\ref{fig:vacuity}): reconstructed on captured reads the bound
exceeds $1$ --- certifies nothing --- on $6.7\%$ of reads at $m{=}2$
masked mantissa bits, $41.8\%$ at $m{=}3$, $92.9\%$ at $m{=}4$; a
worst-key probe returns ``$\mathrm{TV}\le1.54$''. Each read carries a
\emph{valid} certificate, not always an \emph{informative} one.
\begin{figure}[htbp]\centering
\includegraphics[width=0.56\linewidth]{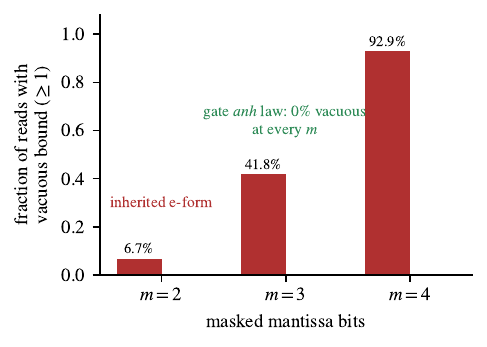}
\caption{Valid is not the same as informative. Fraction of captured
reads on which each per-read bound exceeds $1$ (certifies nothing) as
mantissa bits are masked: the inherited e-form degrades from usable to
useless; the gate $\tanh$ law of \eqref{eq:gatelaw} stays non-vacuous.
This measurement motivated replacing the bound, not post hoc.}
\label{fig:vacuity}
\end{figure}

\emph{Cumulatively} the naive obstacle --- total variation does not
add across reads and layers --- dissolves: what adds is the
output-logit perturbation, because the residual stream is additive, so
the whole request's served TV is one tanh over a \emph{sum} of
per-layer bounds (\texttt{cumulative\_output\_tv}, developed with the
measurement in \S\ref{sec:quantifier} and \eqref{eq:cumtv}).
\emph{Empirically}, per-read attention-distribution shift and per-read
output error move together on real pools (Spearman $0.935$ --- an
ordering statement, not a magnitude one, on the attention leg); it
motivates the pre-registered target for a sound per-layer bound, and is
not evidence that the sum is small. The expert-weight action is the
same shape of open work.

\begin{figure}[htbp]\centering
\includegraphics[width=0.62\linewidth]{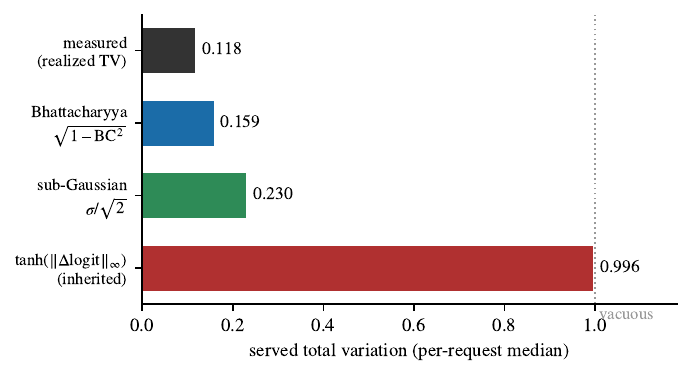}
\caption{Escaping the $\ell_\infty$ wall. The cumulative served total
variation realized by the deployed compressor (top) against three sound
upper bounds, per-request medians over eight requests on the armed
model. The bound inherited from the per-read theory,
$\tanh\|\Delta\mathrm{logit}\|_\infty$, is \emph{vacuous} at $0.996$: it
prices the worst logit coordinate, which sits on a low-probability
token. Total variation instead prices \emph{mass movement}, and both a
mass-overlap bound (Bhattacharyya, $0.159$) and a pure-variance bound
(sub-Gaussian, $0.230$) are non-vacuous and machine-checked. All quantities come from the same verdict artifact as the prose.
\textbf{Caliber:} these are computed \emph{post hoc} from paired exact
and compressed logits, so the figure certifies the \emph{geometry} ---
which functional can bound realized served TV at all --- not a
certificate the online witness emits today; supplying the same bound from
the witness alone is Theorem~\ref{thm:main}'s a-priori branch, whose
open constant we state in \S\ref{sec:limits}.}
\label{fig:ladder}
\end{figure}

\begin{theorem}[Per-step variance budget to request-level served risk]
\label{thm:main}
Let a request's decoding steps be indexed by history $h$; let $p_h$ be
the exact served distribution at $h$ and let the compressed step's
distribution be the tilt $p'_{h,\omega}(v)\propto p_h(v)
e^{\delta_h(\omega,v)}$ induced by that step's rounding draws $\omega$
(the whole per-step draw vector, whatever its internal structure).
Assume, \emph{for every history $h$}:
\emph{(i)} gauge, $\sum_v p_h(v)\delta_h(\omega,v)=0$ for every $\omega$;
\emph{(ii)} a per-coordinate moment bound
$\mathbb{E}_\omega[e^{\delta_h(\omega,v)}]\le B_h(v)$, where writing
$B_h(v)=\exp(b_h(v)+s_h(v)^2/2)$ separates a \emph{deterministic bias}
$b_h(v)$ from a fluctuation proxy $s_h(v)^2$;
\emph{(iii)} a uniform step budget $\mu$ with
$\sqrt{\log\sum_v p_h(v)B_h(v)}\le\mu$.
Then the step loss obeys the mass-weighted bound
\begin{equation}\label{eq:step}
  \mathbb{E}_\omega\big[\mathrm{TV}(p_h,p'_{h,\omega})\big]
  \;\le\;\sqrt{\textstyle\log\sum_v p_h(v)B_h(v)}\;\le\;\mu ,
\end{equation}
and if the ledger admits every step at rate $\lambda>0$ --- i.e.\ the
admission invariant \eqref{eq:admission} of \texttt{cumloss\_admission}
holds,
\begin{equation}\label{eq:admission}
  \frac{\log(1/\delta)+\sum_{s\le t}\log\!\big(1+\mu(e^\lambda-1)\big)}
       {\lambda}\;\le\;B_t
  \qquad\text{for every }t,
\end{equation}
which is what the controller checks before admitting --- then the request-level guarantee \eqref{eq:tail} holds: for the
whole request
\begin{equation}\label{eq:tail}
  \Pr\Big(\exists\,t\le T:\ \textstyle\sum_{s\le t}\mathrm{TV}_s>B_t\Big)
  \;\le\;\delta .
\end{equation}
\end{theorem}

\noindent\emph{Proof (machine-checked).} The statement is
\texttt{request\_tail\_of\_served\_tv\_massweighted}; its Doob-supplied
root, taking bounded differences directly, is
\texttt{request\_tail\_of\_served\_tv\_doob\_massweighted}. The argument
--- sub-Gaussian kernel, mass-weighted Jensen/Fubini step, and the
Bernoulli domination \eqref{eq:bernoulli} whose mean-one factor is what
the anytime ledger consumes --- is spelled out in
Appendix~\ref{app:proof}. Axioms are the three standard ones, no
\texttt{sorry}.
\begin{equation}\label{eq:bernoulli}
  \mathbb{E}\big[e^{\lambda\mathrm{TV}}\big]\;\le\;1+\mu\big(e^\lambda-1\big).
\end{equation}

\section{Instantiating the step budget}
\label{sec:instantiate}
The theorem above is a beam with two ends. This section supplies the
left end --- the per-step moment constant --- through five routes, of
which four failed and each failure is informative.

\emph{Is the budget of the right order?} Substituting one realized
draw for $B_h(v)$ gives $0.230$ (median over the $8$ captured requests)
against a realized served TV of $0.118$ --- a \emph{realization}, not a
bound, but it places the object the theorem prices at the right order
where the system runs; the three steps from there to a certificate are
stated in Appendix~\ref{app:rightorder}.

\emph{Making the budget computable.} The step budget
$\sum_v p_h(v)B_h(v)$ is a sum over $129{,}280$ vocabulary terms per
decoded token, which no serving path will pay --- and does not have to:
for any high-probability set $S$ and tail envelope $B_{\max}\ge B_h(v)$
for $v\notin S$,
\begin{equation}\label{eq:topk}
  \sum_v p_h(v)B_h(v)\;\le\;\sum_{v\in S}p_h(v)B_h(v)
    +\Big(1-\sum_{v\in S}p_h(v)\Big)B_{\max},
\end{equation}
so the $O(|S|)$ right-hand side may be used in hypothesis (iii)
unchanged (\texttt{massweighted\_topk\_bound}, composed in
\texttt{request\_tail\_of\_served\_tv\_topk}); served distributions
concentrate, so the tail mass multiplying $B_{\max}$ is small and the
budget is spent on the tokens that carry the probability --- a per-token
evaluable bound.

\emph{Instantiating the constants.} What remains is to supply the
theorem's moment hypothesis for a real network. We tried five routes and
kept the fifth; the ladder in Table~\ref{tab:ladder} is the honest
summary, and its shape carries a lesson we state once rather than four
times.

\begin{table}[htbp]\centering\small
\setlength{\tabcolsep}{3pt}\footnotesize
\begin{tabular}{@{}llc@{}}
\toprule
Route to the step budget & What it prices & Value\\
\midrule
Global Lipschitz envelope (weights) & worst case over $\mathbb{R}^d$ & $10^{221}$\\
Local operator norm, as a supremum & worst direction, rank $42$ & $\sigma_{\max}\!\approx\!16.0$\\
Per-layer, per-token empirical MGF & assumes layer-local additivity & $0.0390$ (model-dep.)\\
Empirical Bernstein on it & worst-case range $e^{c}$ & $7.64$ (vacuous)\\
\midrule
\textbf{Joint scalar, betting CS} & \textbf{the functional itself} &
$\mathbf{0.0322}$\\
\quad on deployed rounding & \textbf{no surrogate} & $\mathbf{0.0552}$\\
\bottomrule
\end{tabular}
\caption{Five routes to the same number (the surviving one shown at two
operating points). The first four fail for one reason stated four ways: each prices a worst case --- over inputs, over
directions, over a modelling assumption, over a range --- while the
theorem asks for an expectation over rounding draws. The bounded
differences behind rows three and four we did measure ($c_{\max}=3.81$,
per-token, gauge-centred), and they are diagnostics, not budgets: rows
three and four rest on layer-local additivity, which our own saturation
measurement rejects. Only the fifth route --- the last two rows ---
bounds the quantity the theorem consumes.}
\label{tab:ladder}
\end{table}

The failures are informative. A \emph{global} Lipschitz envelope is
sound and hopeless: power iteration over all $43$ layers gives a stack
envelope of $10^{221}$ (median per-layer factor $1.39\times10^{5}$,
largest spectral norm $387$),
optimistic in three separate ways (power iteration lower-bounds the
norm, RMSNorm omitted, experts sampled). Restricting to the region a
request occupies buys a hundred orders of magnitude --- formalized as a
runtime-checkable containment obligation
(\texttt{bdd\_diff\_of\_lipschitz\_on}, \texttt{lip\_iterComp\_on}) ---
but does not suffice as a supremum: the measured ratio's dispersion
puts the effective rank at $42$ of $1.37\times10^{6}$, lifting the
operator-norm estimate to $16.0$, $251\times$ the typical ratio of
$0.0638$. (We first read that typical ratio as a Lipschitz constant,
which was wrong, and retract it.)

Figure~\ref{fig:gamma} shows the measurement that rules out the whole
layer-local route --- worst-case and typical alike. The resolution is
not a sharper worst case. Fixing a history, resampling
\emph{all} layers jointly, and reducing each draw to the scalar
$Z_j=\sum_v p_v e^{\tilde\delta_{j,v}}$ removes the layer-local product,
the union over $129{,}280$ tokens, and the per-layer accumulation of bias
in one move. Better still, the ledger needs $\mathbb{E}[\mathrm{TV}]$
rather than $\mathbb{E}[Z]$, and total variation lies in $[0,1]$
\emph{deterministically}, so a betting-style confidence sequence bounds
it with no appeal to any sample extremum. That it is the \emph{same}
machinery the ledger runs on is not a slogan: the betting factor
\begin{equation}\label{eq:betting}
  g_h(x)=1+\lambda_h(m-x),\qquad \lambda_h\ge0,
\end{equation}
is nonnegative by construction and has mean at most one under
$H_0:m\le\mathbb{E}[X]$, so
Ville's inequality --- the one already formalized for the ledger ---
bounds the probability that the capital ever crosses $1/\delta$
(\texttt{betting\_factor\_eprocess}, \texttt{betting\_ucb\_anytime}).
Inverting the test family built from \eqref{eq:betting} is what defines
the bound, so its anytime coverage rests on a machine-checked argument
rather than on a coverage simulation alone. The number:
$\mathbb{E}_\omega[\mathrm{TV}]\le0.0322$ at $\alpha=0.01$ over $279$
draws, against a realized mean of $1.25\times10^{-5}$ (the scalar-moment
route gives a point estimate of $0.0023$ on the same draws). The
per-token cost is bounded too: enveloping the tail on the product
$p_vB_v$ rather than on $B_v$ alone --- the same anti-correlation that
made the $\ell_\infty$ bound vacuous --- cuts the vocabulary sum to a few
thousand head tokens at a stated price rather than for free: the
tail-enveloped budget is $0.276$ at $K=8192$ against $0.0489$ over the
full $129{,}280$-token vocabulary, a $5.6\times$ envelope that degrades
to $0.967$ at $K=1024$ and to a vacuous $1.288$ at $K=64$. These belong
to the per-layer bounded-difference route we retain as a diagnostic, and
$K=8192$ is the smallest workable setting rather than headroom.

\begin{figure}[htbp]\centering
\includegraphics[width=0.62\linewidth]{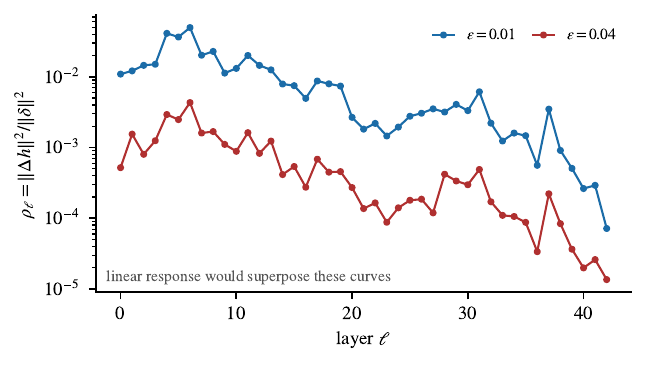}
\caption{The measurement that shaped the theorem. Per-layer propagation
ratios $\rho_\ell=\|\Delta h\|^2/\|\delta\|^2$ from random probes ($43$
layers, noise floor exactly zero). A first-order response would make the
two injection scales superpose; they separate by an order of magnitude
at every layer, and the pre-registered linearity gate fails by a median
factor of $14.1$ ($86$ layer--probe pairs, consistent with a saturated
response: $\rho\propto\varepsilon^{-2}$ gives $16$). We therefore do not
report these as sound propagation constants; Theorem~\ref{thm:main} is
stated so no layer-local assumption enters it --- the measurement rules
out the cheap route to its constants, not the theorem.}
\label{fig:gamma}
\end{figure}

\emph{Removing the surrogate.} Those draws were injected perturbations;
the deployed rounding can be sampled directly, because the rounding seed
mixes in the request identifier --- repeating a prompt with radix
caching disabled draws fresh rounding each time. Four pre-registered
gates decide whether such a run means anything: compression must occur
($1{,}999{,}586$ entries did), the draws must vary (all $160$ realized
total variations distinct --- coincidence would have meant re-running
one draw, reported as failure), the arms must reconcile, and the bound
must be non-vacuous. It is:
\begin{equation}
\mathbb{E}_\omega[\mathrm{TV}]\;\le\;0.0552
\qquad\text{vs.}\qquad
\text{realized mean}\;=\;1.64\times10^{-5},
\label{eq:nonvacuous}
\end{equation}
a factor of $3{,}370$ of slack, and the sample-size reading of that
slack is controlled rather than inferred: an earlier pass of the same run at $56$ draws gave
$0.1463$ with an essentially unchanged mean, so the constant tracks the
sample and not the quantity. This is a statement about the deployed
compressor rather than a stand-in for it.

\section{The quantifier: extrapolating across histories}
\label{sec:quantifier}
Everything so far holds \emph{at a history}. The theorem's hypothesis is
quantified over every history a request visits, and that gap is where the
remaining distance to a deployed certificate lives. We measure it, and
then close it on a stated population.

\emph{One history is not enough --- we measured how far from enough.} We repeat the design across $80$ histories drawn
\emph{without replacement} from a pool of $148$ candidates (forty
deployed-rounding draws each; $60{,}706{,}357$ compressed entries; the
exact arm bitwise identical across repeats). The sampling design is not
cosmetic: drawing without replacement from a stated pool makes the
histories \emph{exchangeable} by construction --- exactly the
hypothesis the extrapolation bound needs, and one our first attempt
(histories enumerated at increasing lengths) did not satisfy.

Figure~\ref{fig:hist} shows the outcome of the eighty-history design.
Per-history mean total variation ranges from $2\times10^{-6}$ to
$2.6\times10^{-1}$: a factor of $133{,}849$. The single-history number
above sits at the insensitive extreme of that range and is not
representative of anything.

\begin{figure}[htbp]\centering
\includegraphics[width=0.78\linewidth]{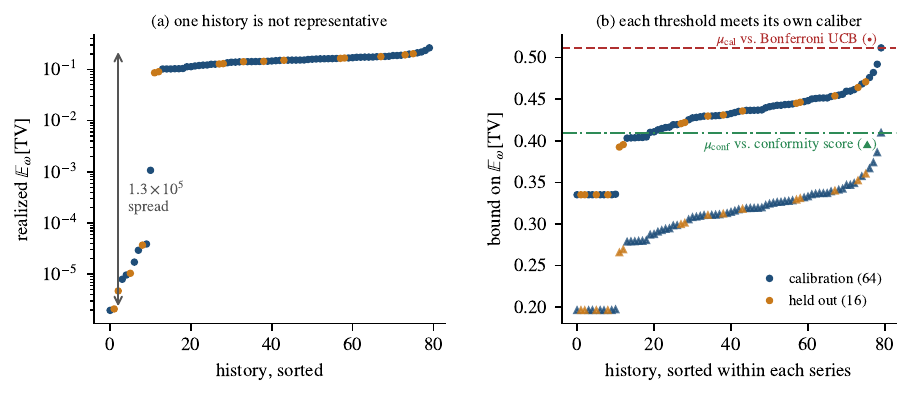}
\caption{Eighty histories, one certificate each. \textbf{(a)} Realized
mean served-output shift, sorted: five orders of magnitude separate the
extremes, so the choice of history decides a single-history answer.
\textbf{(b)} Certificates, each threshold against \emph{its own}
caliber: $\mu_{\mathrm{cal}}$ vs.\ Bonferroni-corrected upper bounds
($\bullet$, $\alpha'=1.25\times10^{-4}$), $\mu_{\mathrm{conf}}$ vs.\
conformity scores ($\blacktriangle$, $\alpha=0.01$, no Bonferroni owed).
Every held-out history clears both; held-out points never entered
either threshold.}
\label{fig:hist}
\end{figure}

The calibration procedure now certifies transfer. With a split fixed
in advance ($64$ calibration, $16$ held out), the certified reading is
three-valued --- a held-out history \emph{certifies} when its upper
bound falls below $\mu_{\mathrm{cal}}$, is \emph{certified to violate}
only when its \emph{lower} bound exceeds it, and is otherwise
\emph{undetermined} (comparing a noisy sample mean against a confidence
bound decides nothing, which is how we first concluded otherwise). All
$16$ held-out histories certify, and the result survives charging the
draw-level confidence honestly: a Bonferroni split puts each history at
$\alpha'=1.25\times10^{-4}$ and loosens $\mu_{\mathrm{cal}}$ to
$0.5116$.

Counting those outcomes gives a Clopper--Pearson bound of $0.2501$ on
the rate at which a further history fails to certify --- the weakest
instrument available for the job (the full risk decomposition
\eqref{eq:risk} and its sample-size arithmetic are in
Appendix~\ref{app:cprisk}); the conformal route below replaces it.

\paragraph{The counting instrument was the bottleneck: a conformal replacement.}
That $0.2501$ is not what the data can support; it is what \emph{our
instrument} could extract. Counting held-out pass/fail discards the
order structure among per-history scores --- the only structure
exchangeability needs. The order statistic uses it directly: for $N+1$
exchangeable scores the events ``score $i$ is the strict maximum'' are
pairwise disjoint and equiprobable, so
\begin{equation}\label{eq:conformal}
\Pr\bigl(S_{\text{new}}>\max_{i\le N}S_i\bigr)\;\le\;\tfrac{1}{N+1},
\end{equation}
and a single union bound on top of \eqref{eq:conformal} for the new
history's own coverage gives
$\Pr(\mathbb{E}_\omega[\mathrm{TV}\mid h_{\text{new}}]>\mu_{\mathrm{conf}})
\le 1/(N+1)+\alpha$ (Lean: \texttt{conformal\_risk\_union}; disjointness
is derived, exchangeability enters as an explicit hypothesis).
On the identical $80$ histories this improves both numbers, but the two
routes do not deliver the same \emph{kind} of guarantee and we will not
quote a ratio across kinds. The conformal $0.0254$ is a \emph{marginal}
bound: the probability is averaged over the calibration draw as well as
the new history. The Clopper--Pearson $0.2501$ is a \emph{high-confidence}
statement about the calibration set we actually drew. Comparing them
directly would repeat, one level up, the caliber swap this paper is
about.

The like-for-like comparison exists and we make it, matching the
instrument to the sampling design: drawn without replacement from a
finite pool, the textbook $\mathrm{Beta}(1,N)$ training-conditional
bound does not apply (it assumes i.i.d.\ draws from a continuum), and
the right object is the negative hypergeometric. With $M=148$ pool
members and $N=64$ calibration draws,
$\Pr(J\ge j)=\binom{M-j}{N}/\binom{M}{N}$ for the number $J$ of pool
members exceeding the calibration maximum, so with probability
$\ge0.99$, $J\le7$, and a further draw misses coverage with probability
at most $7/84=0.0833$ ($0.0933$ with the new history's own $\alpha$
added) --- a $99\%$-confidence statement of the same kind
Clopper--Pearson makes, a factor of $2.7$ against its $0.2501$, and the
number we claim. (The marginal $0.0254$ is the right figure for
long-run average risk across recalibrations, the wrong one for today's
calibration set.) The threshold itself tightens, from
$\mu_{\mathrm{cal}}=0.5116$ to $\mu_{\mathrm{conf}}=0.4095$:
calibration scores are conformity scores, not certificates, so the
per-history Bonferroni split disappears --- a fifth of the histories
had been spent on a held-out check, and the rest charged twice.

\begin{figure}[htbp]\centering
\includegraphics[width=0.62\linewidth]{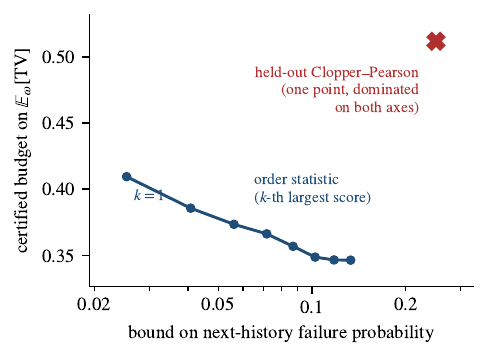}
\caption{Same measurements, two instruments. Counting held-out
pass/fail yields a single point; the order statistic yields a frontier,
and the counting point is dominated on \emph{both} axes. Moving along the
frontier trades a looser certified budget for a smaller failure
probability at no extra measurement.}
\label{fig:frontier}
\end{figure}

Exchangeable by construction, this is a bound on a stated population
--- the uniform distribution over the $148$-candidate pool --- and we
test it rather than assert it: the $16$ held-out histories are further
draws from the same scheme, none exceeds $\mu_{\mathrm{conf}}$
(expected $\le0.25$ exceedances; the threshold never saw them).
Figure~\ref{fig:frontier} puts the two instruments in one plane: the
paper's recurring diagnosis turned on its own analysis, and the one
instance where it \emph{buys} a claim rather than killing one. What
the bound does \emph{not} cover is runtime traffic. The pool is built,
and built narrowly ($148$ truncations of a few base texts --- prefixes,
not the history stream a request walks), and $16$ test points against
$0.25$ expected exceedances make the held-out check a sanity check, not
a powered confirmation. Closing that gap needs exchangeable sampling
from real traffic, not a sharper inequality. What the order statistic changes is the cost of eventually getting
there: at the marginal caliber a $1\%$ history sub-budget needs
$N=199$ exchangeably sampled histories (uneven splits are cheaper
still), $2.3$--$5.3\times$ fewer than the zero-event Clopper--Pearson
route at the same total; the accounting, the request-level budget
constraint \eqref{eq:budget}, and Figure~\ref{fig:histcost} are in
Appendix~\ref{app:histcost}.

What the experiment establishes is a certified cross-history bound on
a stated population, premise met by construction and checked on
held-out draws; what it does not establish is that the population is
the one a deployed system faces, or that a runtime membership test and
fallback exist. And the run fixes a prompt and a prefill-final position,
varying only the rounding, so ``anytime-valid'' here means valid as
draws accumulate --- not along the histories an autoregressive request
visits, which is what the theorem quantifies over. Lifting to runtime
histories is the remaining structural step (a held-out design over
domains, lengths and decode positions, a per-bucket bound, a fallback
on leaving the calibrated set); the risk then reads as in
\eqref{eq:risk}, of which this paper accounts for the first two terms.
Independence of draws across requests is assumed throughout.

\emph{What supplies hypothesis (ii).} Two routes, same constant.
Tensorization needs each layer's contribution to depend on its own draw
alone --- the idealization our propagation probe falsifies. A Doob
decomposition needs only bounded differences and carries the bias
explicitly, $\mathbb{E}[e^F]\le\exp(\mathbb{E}[F]+\sum_\ell c_\ell^2/8)$
(\texttt{doob\_mgf\_le\_biased}), no layer-local assumption; joined to
the composition's shape by \texttt{condE\_eq\_prod\_sum} and
\texttt{doob\_mgf\_prod\_sum}, the whole path is one theorem whose
inputs are statements about the model and whose output is the
request-level tail bound
(\texttt{request\_tail\_of\_served\_tv\_doob\_massweighted}; one
\texttt{\#print axioms} returns the three standard axioms and nothing
else). What remains is not a proof step but an \emph{instantiation}:
the constants $c_\ell$ in served-logit units from what the witness
records --- and our propagation measurement rules out the first-order
surrogate. That is the open problem, and we state it as the paper's
main one.\medskip

\paragraph{The cumulative tier, measured and soundly bounded.}
Running the armed model twice on eight fixed prompts --- exact KV vs.\
deployed compression --- and diffing the final next-token logits
($\mathrm{E_{out}}=\|\Delta\mathrm{logit}\|_\infty$ on the captured
pre-head hidden state), the realized whole-request served TV is small:
median $0.118$, worst $0.22$ (Figure~\ref{fig:ladder}) --- the
cumulative served effect, end-to-end, \emph{non-vacuous}. What sound
bound certifies it? The obvious $\tanh(\mathrm{E_{out}})$ is vacuous
($0.996$): $\mathrm{E_{out}}$ is a worst-\emph{coordinate} quantity
(median $3.1$, sitting on a low-probability token), and total variation
cares about \emph{mass movement}, not the worst coordinate. The
Bhattacharyya coefficient $\mathrm{BC}=\sum_v\sqrt{p_v p'_v}$ gives, by
Cauchy--Schwarz on $\sqrt{p}\pm\sqrt{p'}$, the machine-checked bound
\begin{equation}
\mathrm{TV}(p,p')\;\le\;\sqrt{1-\mathrm{BC}(p,p')^{2}},
\qquad
\mathrm{BC}(p,p')=\sum_v\sqrt{p_v\,p'_v}
\label{eq:tvbc}
\end{equation}
(\texttt{tv\_le\_hellinger}) --- and \emph{this} sound bound is
non-vacuous: median $0.16$, worst $0.28$, zero violations, within
$1.4\times$ of the measurement. So the cumulative served TV is not only
measured small but bounded by a machine-checked sound certificate that
breaks the $\ell_\infty$ wall. 

A second machine-checked bound settles which regime the perturbation
lives in. Reading the sub-Gaussian proxy off the realized perturbation
($\sigma^2=2\ln\mathbb{E}_p[e^{\tilde\delta}]$, the tightest the
measured MGF admits), \texttt{served\_tv\_le\_subgaussian} certifies
$\mathrm{TV}\le\sigma/\sqrt2$, evaluating to $0.23$ (median, zero
violations, within $\sqrt2$ of Hellinger --- the price of variance
alone). The decisive ratio $\sigma^2/\mathrm{Var}_p(\delta)$ has
per-request median $0.99$ (range $0.78$--$1.47$): \emph{the realized
perturbation is variance-tight --- sub-Gaussian, not range-dominated
--- despite} $\mathrm{E_{out}}$ reaching $3.1$--$4.3$ on the served
position ($18.9$ at the worst prefill position). The large logits sit
on low-probability tokens, so under $p$ the MGF is governed by the
variance ($\mathrm{Var}_p(\delta)\approx0.115$), not the range ---
``mass movement, not worst coordinate,'' read off the served softmax
itself, falsifying the reading that a draw's vocabulary spread must
behave as a range.

That measurement reshapes the a-priori target rather than closing it.
The a-priori bound --- from the witness budget alone, with no exact
reference run --- has its kernel machine-checked: a $p$-centred,
$\omega$-sub-Gaussian perturbation with proxy $s^2$ is within
$\mathrm{TV}\le s/\sqrt2$ (\texttt{served\_tv\_le\_subgaussian}), a pure
variance bound with no worst-case range. Supplying that proxy from the
witness involves two distinct randomnesses --- the kernel's is a
vocabulary-softmax moment, the martingale's is over rounding draws
$\omega$ --- so the guarantee is inherently one about
$\mathbb{E}_\omega$, which is exactly the epistemics of the witness
e-process. With the vocabulary moment measured variance-tight, what the
proxy needs is a variance-scale bound on
$\mathrm{Var}_p(\delta)=\Delta h^\top\mathrm{Cov}_p(\mathrm{head})\Delta h
\le\mathrm{tr}\,\mathrm{Cov}_p(\mathrm{head})\|\Delta h\|^2$
(\texttt{var\_p\_delta\_le\_trace\_cov}), with $\|\Delta h\|^2$ carried
back by the residual stream: writing $\Delta h=\sum_\ell X_\ell$, the
increments are uncorrelated over $\omega$, so
\begin{equation}\label{eq:varadd}
  \mathbb{E}_\omega\|\Delta h\|^2=\sum_\ell\mathbb{E}_\omega\|X_\ell\|^2
\end{equation}
--- by \eqref{eq:varadd} variances \emph{add}, they do not accumulate as
magnitudes (\texttt{residual\_second\_moment\_le}) --- which is why a witness that
sums per-layer \emph{variances} is the right budget. The
uncorrelatedness is not assumed: modelling the draws as a product measure
discharges it by Fubini
(\texttt{crossSum\_indep\_meanzero\_eq\_zero}), and the per-layer
propagation coefficients are bounded by a Frobenius envelope whose input
norm the score bridge ties back to the witness
(\texttt{linear\_propagation\_frobenius},
\texttt{score\_perturbation\_l2\_le}). The $C_\ell$ here is the variance
\emph{after} propagation to the logits; relating it to the KV-side
witness is the propagation constant we take up next.

We measured those propagation constants on the armed model, and the
measurement returned a negative result worth reporting. Injecting a
random perturbation at each layer's attention output and reading the
served-position hidden state gives $\rho_\ell=\|\Delta
h\|^2/\|\delta\|^2$ with layer median $4.07\times10^{-3}$ (range
$7.2\times10^{-5}$ to $4.9\times10^{-2}$ over all $43$ layers, two probe
directions each, zero-valued noise floor --- repeated identical requests
are bitwise deterministic), which by the Hutchinson identity puts
$\|J^\ell\|_F^2$ near $5.6\times10^3$. But the pre-registered linearity
gate fails: doubling-and-doubling the injection ($\varepsilon$ ratio
$4$, hence energy ratio $16$) leaves $\|\Delta h\|^2$ essentially
unchanged, so $\rho$ drops by a median factor of $14.1$ instead of
staying flat. A response that is independent of the injected energy is
saturated, not linear. Empirically, then, the first-order surrogate
$X_\ell=J^\ell\delta a^\ell$ does \emph{not} hold for this MoE at the
perturbation scales we can inject --- most plausibly because expert
selection is an $\arg\max$ and flips discontinuously. We therefore do
\emph{not} report these $\gamma_\ell$ as sound propagation constants;
they are a diagnostic that falsifies the linearization, and they say the
right repair is not a better Jacobian estimate but a decomposition that
never assumes layer-local dependence in the first place --- a Doob
decomposition $D_\ell=\mathbb{E}[F\mid\mathcal{F}_\ell]-\mathbb{E}
[F\mid\mathcal{F}_{\ell-1}]$, whose conditional MGF bound tolerates the
true nonlinear, adaptive dependence. We take that step in the two
places it is cheap and load-bearing.

First, the Doob form. Replacing the falsified layer-local hypothesis
with a Doob decomposition costs nothing in the constant: with
\begin{equation}
D_\ell\;=\;\mathbb{E}[F\mid\mathcal{F}_\ell]
          -\mathbb{E}[F\mid\mathcal{F}_{\ell-1}] ,
\label{eq:doob}
\end{equation}
and assuming only \emph{bounded differences} --- changing the $\ell$-th
draw moves $F$ by at most $c_\ell$ --- the same conditional-MGF
machinery gives
\begin{equation}
\mathbb{E}\bigl[e^{F}\bigr]\;\le\;
\exp\Bigl(\mathbb{E}[F]+\tfrac{s^2}{2}\Bigr),
\qquad
s^2=\sum_\ell \frac{c_\ell^{2}}{4} ,
\label{eq:doobmgf}
\end{equation}
--- the bias carried explicitly, because
stochastic rounding being unbiased \emph{per layer} does not make the
final logit unbiased once routing flips and activations intervene ---
identical proxy shape and constant, but now $F$ may depend on the whole
draw sequence nonlinearly and adaptively (\texttt{doob\_mgf\_le}). The
modeling assumption, not the bound, is what weakens.

Second, the bridge to a request-level budget. Our served-TV result is an
\emph{expectation} bound for one fixed output; a risk budget needs a tail
bound over a request's stream. The missing composition is short and we
machine-check it: if a per-step loss satisfies $0\le L_t\le1$ with
conditional expectation $\mathbb{E}[L_t\mid\mathcal{F}_{t-1}]\le\mu$, then
convexity of $t\mapsto e^{\lambda t}$ on $[0,1]$ gives the Bernoulli
domination $\mathbb{E}[e^{\lambda L_t}\mid\mathcal{F}_{t-1}]\le
1+\mu(e^\lambda-1)$ (\texttt{bernoulli\_mgf\_le}) --- an
\emph{expectation}-only hypothesis, no range bound needed beyond
$[0,1]$. Taking $\psi=\log(1+\mu(e^\lambda-1))$ makes the e-process
factor have mean at most one
(\texttt{eprocess\_factor\_le\_one\_of\_expected}), which is exactly the
hypothesis our anytime ledger consumes, and composing yields
\begin{equation}\label{eq:beam}
  \Pr\Big(\exists\,t\le T:\ \textstyle\sum_{s\le t}L_s>B_t\Big)\;\le\;\delta
\end{equation}
(\texttt{request\_tail\_of\_expected\_loss}). So the beam \eqref{eq:beam} from expected loss to request-level tail is
in place and machine-checked.
Its left end is proved too, and by a modeling choice rather
than new machinery: take the ledger's event to be one decoding step's
\emph{entire} vector of rounding draws. The one-step conditional
expectation \eqref{eq:predictable} requires is then literally the
quantity our served-TV theorem bounds, evaluated at each history, so
$\mathbb{E}[L_t\mid\mathcal{F}_{t-1}]\le\mu$ holds pointwise
(\texttt{request\_tail\_of\_served\_tv\_massweighted}). The ledger
therefore closes on our theorem, not on a hypothesis. The name matters:
the earlier \texttt{request\_tail\_of\_served\_tv} does carry a
layer-local hypothesis --- each layer's contribution depends on its own
draw --- which our propagation probe falsifies; the mass-weighted form
that Theorem~\ref{thm:main} states carries none, its moment condition
being written on the whole per-step draw vector. What remains open is
instantiating its constants, not repairing its form.

What is left is the model-specific numerics --- the actual propagation
constants $\|J^\ell\|_F^2$ and the first-order linearization
$X_\ell=J^\ell\delta a^\ell$ under which the increments are
$\omega$-local --- not an intractable worst-case wall. Locating this precisely,
machine-checking the entire propagation calculus (orthogonality
discharged, Pythagorean identity, Frobenius envelope, score bridge, and
the sound vocabulary-to-$\omega$ transfer), and measuring that the served
perturbation lives on the tractable side of it, is what we contribute;
instantiating the constants is open --- and the measurement itself is
single-step prefill, not a full multi-step decode.

\section{From the witness to served output, per read}
\label{sec:perread}

\paragraph{A gate-parameterized served-TV law (per read).}
The physical gate admits a KV write only when its certified witness
upper bound satisfies $u_e\le w_{\mathrm{thr}}$, and the audited
realization never exceeds $u_e$, so \emph{every admitted read carries
witness $\le w_{\mathrm{thr}}$ by construction}. Composing this gate
fact with the Cauchy--Schwarz score bridge
$|\Delta\mathrm{score}|\le\mathrm{scale}\cdot\|q\|\cdot\|\Delta k\|$
and a \emph{tightened} softmax bridge yields a closed-form ceiling on
each admitted read's attention TV. The softmax bridge is the crux: the
inherited exponential-tilt form $\mathrm{TV}\le\tfrac12(e^{2\varepsilon}-1)$
is vacuous at the perturbations at hand ($\varepsilon\approx0.70$ gives
$1.54$), but a moment/chord argument gives the machine-checked,
uniformly tighter $\mathrm{TV}\le\tanh\varepsilon$ (numerically the
exact supremum is $\tanh(\varepsilon/2)$, left as a conjecture). With
it the per-read ceiling and its inverse become
\begin{equation}\label{eq:gatelaw}
  \mathrm{TV}\le \tanh\!\bigl(a_q w_{\mathrm{thr}}\bigr),
  \quad
  w_{\mathrm{thr}}^\star=\operatorname{atanh}(\tau^\star)/a_q,
  \quad a_q=\max_{\text{layers}}\mathrm{scale}\cdot\|q\|,
\end{equation}
--- both machine-checked (\texttt{served\_tv\_le\_of\_gate\_tanh},
\texttt{gate\_threshold\_for\_sla}) --- the first time the controller's
knob is tied to a served-output quantity by proof rather than by the
empirical $0.935$. An ungated probe campaign supplies one
self-consistent bridge measurement ($a_q=1.0034$, sampled worst-key
witness $0.70$): reading $0.70$ into the law gives $\tanh(0.70)=0.61$
($0.34$ under the tight conjecture), non-vacuous where the e-form gave
$1.54$ --- but that is a \emph{sampling} instantiation, the very
substitution of a sample maximum for a gate constant this paper
diagnoses elsewhere. The armed deployment's actual threshold is
$w_{\mathrm{thr}}=35.34$ (\texttt{WITCERT\_LEDGER\_WTHR}), where the
same formula returns $1.000$: vacuous.

We first recorded this as ``the certified instantiation is open, pending
an $a_q$ measured on the armed model,'' and that was the wrong diagnosis.
The \emph{sound} $a_q$ needs no run at all. It is
$\mathrm{softmax\_scale}\cdot\|q\|_{\mathrm{static}}$, where
$\|q\|_{\mathrm{static}}=\|\gamma\|_\infty\sqrt{r}\max_h\sigma_{\max}
(W_q^{(b),h})$ holds for \emph{every} input and
$\mathrm{softmax\_scale}=d^{-1/2}$ is a configuration constant, so it is
same-model with the armed gate by construction. Evaluating it on the
armed model gives $a_q^{\mathrm{sound}}=1.5067$ --- only
$1.50\times$ the sampled $1.0034$. The bound was never the problem:
\begin{equation}
\underbrace{w_{\mathrm{thr}}^\star=0.0332}
      _{\text{needed for }\tau^\star=5\%\ (\text{sound})}
\qquad\text{vs.}\qquad
\underbrace{w_{\mathrm{thr}}=35.34}_{\text{deployed}}
\;\;=\;\;1064\times ,
\label{eq:gategap}
\end{equation}
and of that $1064\times$, the query envelope's slack accounts for
$1.50\times$. Tightening $\|q\|$ --- the open item we had named --- cannot
close a gap three orders of magnitude wide.

One bound-side lever remains, and we charge it rather than let the
previous sentence overstate the case. The Cauchy--Schwarz step behind
$a_q$ takes the worst point of a \emph{ball} of radius $W$; replacing
the ball by the measured error ellipsoid is the same move the
companion paper makes for routing, where it bought $3.9$--$18.5\times$,
and its ceiling here is the alignment factor
$\sqrt{d}=\sqrt{512}\approx22.6$. Charged on the armed model's own
captures, it buys nothing:
\begin{equation}
\underbrace{1.50\times}_{\|q\|\text{ bound}}
\;\cdot\;
\underbrace{0.89\times}_{\text{ball}\rightarrow\text{ellipsoid, measured}}
\;\cdot\;
\underbrace{\approx700\times}_{\text{gate operating point}}
\;=\;1064\times .
\label{eq:threelayer}
\end{equation}
The middle term is below one at the median and never usefully above it
($0.39$--$1.13$ across the $43$ layers) --- the ellipsoid bound is
\emph{looser} than the ball it replaces --- with zero violations on the
held-out half,
so this is not an artefact of overfitting the covariance. It is also not
the trivial case of isotropic error: the stable rank of $\operatorname{Cov}
(\Delta k)$ is $37$ of $448$, so the energy really is concentrated. What
defeats the ellipsoid is its radius. The bound is $r\|C^{1/2}q\|$, and
$r$ is an extreme-value statistic in $448$ dimensions --- it must cover
the most outlying row along the \emph{low}-variance directions, where
$C^{-1}$ is largest. Measured, $r\in[25.5,31.1]$, above the isotropic
reference $\sqrt{448}=21.2$: the radius gives back everything the
anisotropy saves. On the routing side the same substitution replaces a
per-expert $\ell_\infty$ ball rather than an $\ell_2$ one and pays off
--- a reminder that a technique's payoff does not transfer across bound
structures, and the reason we measured here rather than extrapolated. With all three
layers now measured, ``the residual is the operating point'' is an
observation rather than a conjecture. So the honest status is sharper than
``instantiation is open'': the law is proved, its sound instantiation is
\emph{in hand}, and it says the deployed gate is nowhere near a per-read
served SLA. Pillar~A meets, in the per-read caliber, the same usability
question that the companion paper settles for pillar~B --- with the
difference that here the certified object is the right one, and only the
operating point is wrong. The design law then reads a served SLA straight off the gate:
certifying $\tau^\star=5\%$ per-read TV needs
$w_{\mathrm{thr}}^\star\le0.05$ in those witness units. The bound is not
merely theoretical: reconstructing exact-versus-compressed attention on
reads captured from the armed model (a \emph{synthetic} pool ---
same-layer rank-0 reads merged without request identity, RoPE
approximated), the machine-checked $\tanh$ bound holds on \emph{every}
sampled read/draw pair --- zero violations, the tight
$\tanh(\varepsilon/2)$ included --- while realized per-read served TV at
the deployed 2-bit setting is small (median $0.03$, worst $0.29$). The
inherited e-form is vacuous ($\ge1$) on $0.07$ of these reads; $\tanh$
never is. Two honest limits remain: $a_q$ here is a sampling maximum
(the sound static envelope above is the backstop), and this measures
the per-read \emph{attention} tier --- the cumulative final-logit
perturbation is measured in \S\ref{sec:quantifier} (served TV median
$0.12$, soundly bounded by $0.16$).

\begin{table}[htbp]\centering\small
\setlength{\tabcolsep}{5pt}
\begin{tabular}{@{}p{0.30\linewidth}p{0.30\linewidth}p{0.32\linewidth}@{}}
\toprule
\textbf{Proved} & \textbf{Measured} & \textbf{Open, declared}\\
\midrule
per-step moment $\to$ request tail, \eqref{eq:step}--\eqref{eq:tail};
Doob form with no layer-local hypothesis;
Bernoulli domination \eqref{eq:bernoulli};
conformal extrapolation \eqref{eq:conformal}
&
$\mathbb{E}_\omega[\mathrm{TV}]\le0.0552$ on deployed rounding;
$16/16$ held-out histories certify;
$133{,}849\times$ cross-history spread;
$+3.8\%$ judgment cost
&
propagation constants (probe \emph{falsifies} the first-order
surrogate); $a_q$ and $w_{\mathrm{thr}}$ from one gated run;
population $=$ our pool, not runtime traffic\\
\bottomrule
\end{tabular}
\caption{What this paper does and does not claim. The middle column is
evidence, not proof; the right column is what a deployed certificate
still needs. Keeping these three apart is the discipline the paper argues
for, and four of our own retractions came from letting them blur.}
\label{tab:status}
\end{table}

\paragraph{Status of the claim.} Table~\ref{tab:status} summarises it.
Established: the machine-checked kernel; the admission mechanism live
in shadow, where the model-validation e-process leaves the radius model
unrefuted over long real streams ($706{,}909$ factors across eight
1{,}024-step decodes; peak $\log M=-3.70$ against a crossing threshold
of $4.61$) --- that, and only that, is what the shadow numbers certify;
the cumulative-loss theorem; the dual-accounting replay; and the
physical-gate armed run (anytime ledger never crossed, budget
dose-response, $+3.8\%$ overhead) with its held-out confirmatory
replication (frozen $\lambda$ and budget, topic-disjoint prompts,
coverage gain same-round $0.14$ vs.\ $0.30$), which promotes the
working point from exploratory to confirmed. Open and declared: pricing
the physical witness account in served-output units. \emph{Per read}
that is now the machine-checked $\tanh$ law; \emph{cumulatively} the
additivity obstacle is dissolved --- the residual stream makes the
output-logit perturbation additive, so the whole request's served TV is
one tanh over a \emph{sum} of per-layer bounds,
\begin{equation}
\mathrm{TV}_{\text{request}}\;\le\;
\tanh\Bigl(\textstyle\sum_\ell b_\ell\Bigr)
\qquad(\texttt{cumulative\_output\_tv}),
\label{eq:cumtv}
\end{equation}
rather than a product of per-layer factors --- which is what made the
aggregate tier reachable at all. What remains open at both tiers is a sound
\emph{supremum} for the query/Jacobian envelope (the $a_q$ and $b_\ell$
factors), today measured rather than bounded --- and, for the Jacobian
factor specifically, our own measurement falsifies the first-order
surrogate it would quantify (\S\ref{sec:instantiate}), which is why the
Doob/bounded-difference form is the route we take --- together with the
graphs-on production fast-path overhead and the MoE expert-weight
action. Until those land, pillar~A is a proved engine with
a validated substrate and a live physical controller --- not yet a
served-output guarantee a user can price.

\section{Consequences for practice}
\label{sec:consequences}
The first instruction has a five-arm replication on the serving stack
itself. Running the RULER-style needle suite over five arms --- an FP8
safe endpoint, a fully aggressive no-gate endpoint, and three gated
configurations sharing one write threshold --- with three independent
server restarts per arm, every arm scores $1.000$ under the
provably-clean caliber (zero evictions and zero ring-page recycling,
both directly counted, which jointly rule out cross-request slot reuse)
and zero authorization violations (Table~\ref{tab:fivearm}). Quality
does not separate the arms; the risk ledger does. The union-budgeted
gate burns $99.4\%$ of its request budget on every twelve-document
batch and pays a $59.3\%$ write-side fallback rate; the
cumulative-loss gate at the same threshold holds a bounded account ---
peak charge $99.97\%$ of budget, with $45$ of $160$ admission calls
budget-forced to exact --- while paying only $56.3\%$: the account
that survives adaptivity also buys $3$ points more coverage at equal
quality. A shadow-audit arm reconciles the packed store against FP8
entry-by-entry ($6.6\%$ mean relative residual), pricing what the
compressed bytes actually hold.

\begin{table}[htbp]\centering\small
\setlength{\tabcolsep}{4pt}
\begin{tabular}{@{}lcclc@{}}
\toprule
Arm & acc & clean & risk ledger & fallback\\
\midrule
FP8 safe (endpoint) & $1.000$ & \ding{51} & none (no compression) & ---\\
aggressive, no gate (endpoint) & $1.000$ & \ding{51} & \textbf{none} & ---\\
gated, union budget & $1.000$ & \ding{51} & exhausts ($99.4\%$ burned) & $59.3\%$\\
gated, cumulative (phys) & $1.000$ & \ding{51} & bounded ($99.97\%$ peak) & $56.3\%$\\
gated + shadow audit & $1.000$ & \ding{51} & as union + reconciliation & $57.5\%$\\
\bottomrule
\end{tabular}
\caption{Five-arm equal-quality table on the serving stack
(DeepSeek-V4-Flash, tp8, three independent restarts per arm, medians;
zero authorization violations anywhere). \textbf{Caliber:} the two
endpoints bound the axis and are not co-axial with the gated arms; the
gate acts on the write-side mantissa-mask layer while the packed
compression itself runs identically in all four compressed arms;
quality is stated only under the provably-clean caliber (zero
evictions and zero ring-page recycling, both directly counted);
fallback is reported as a rate because absolute event counts vary with
batch composition. Throughput is a different caliber (concurrent
clients) and is deliberately not a column of this table --- and we
disclose why: under mixed concurrent workloads (a benchmark warm-up
preceding evaluation) the prototype's c4-tier sparse-read path
exhibits a localized defect (deep-position retrieval degrades while
stored content verifies intact against the reference pool), traced to
unresolved key-space translations and under repair. The quality column
here is measured without such preceding traffic and is unaffected.}
\label{tab:fivearm}
\end{table}

Two instructions survive to the practitioner from this paper's own
scope, each carried by a measurement. First, budget risk with the
anytime-valid account, not the union bound: the union budget exhausts on
every long request while the physically-accounted gate holds its bound
at every admission call and, in the held-out confirmatory round, halves
the exact-fallback rate ($0.30\to0.14$) at matched risk --- coverage is
bought at a price the account states, not found for free. Second, spend
capacity where the workload can convert it: the $3.73\times$ KV
capacity that precision decisions bank converts to $1.19\times$
throughput only under session reuse, and to nothing under one-shot
traffic (Table~\ref{tab:convert}); the decision to buy is a workload
property, not a systems constant. The routing-side prescriptions ---
what not to certify, what not to preserve, and where the freed budget
should go --- are collected with their measurements in the companion
paper's prescriptive table.

\section{Related work}
\label{sec:related}
Four properties would have to hold at once for a precision decision to be
both trustworthy and deployable: it must be made \emph{per request}, it
must carry a \emph{sound} statement rather than an average one, that
statement must survive \emph{adaptive} decoding (where earlier
quantized state feeds later queries), and it must run \emph{inside} a
serving stack. Table~\ref{tab:related} reads the literature against
those four. The pattern is that they never co-occur --- offline work
buys soundness by giving up request conditioning, dynamic-precision
serving buys request conditioning by giving up the guarantee, and the
certified line that has both is confined to the KV domain and budgets by
union bound.

\begin{table}[htbp]\centering\small
\setlength{\tabcolsep}{5pt}
\begin{tabular}{lcccc}
\toprule
& Per- & Sound & Adaptive- & In serving\\
Line of work & request & (not avg.) & valid & stack\\
\midrule
PTQ / second-order weighting
  & \ding{55} & \ding{51} & --- & \ding{55}\\
\quad\emph{GPTQ, AWQ, QERA, YAQA, WaterSIC} & & & &\\
MoE quantization (expert/router-aware)
  & \ding{55} & \ding{55} & --- & \ding{55}\\
\quad\emph{MoEQuant, GEMQ} & & & &\\
Dynamic-precision serving
  & \ding{51} & \ding{55} & \ding{55} & \ding{51}\\
\quad\emph{MorphServe, DP-LLM, QAQ} & & & &\\
Certified runtime quantization (KV)
  & \ding{51} & \ding{51} & \ding{55} & \ding{51}\\
\quad\emph{Calver, WitCert} & & & &\\
\midrule
\textbf{This paper} (served-output budget)
  & \ding{51} & \ding{51} & \ding{51} & \ding{51}\\
\bottomrule
\end{tabular}
\caption{Why the gap is structural rather than incremental. ``Adaptive-valid''
means the guarantee survives the feedback of quantized state into later
queries --- a union budget over a pre-declared event count does not, which
is the deficit \S\ref{sec:av} removes. ``---'' marks lines that make no
per-step claim, so adaptivity does not arise. The routing column is
absent by design: no prior line certifies routing, and
\S\ref{sec:router} is our evidence that certifying it would not have
helped.}
\label{tab:related}
\end{table}

\paragraph{Certified runtime quantization (KV domain).} Runtime
certificates with fallback exist for quantized
attention~\cite{calver2026} and for KV-cache quantization with request
budgets and Lean-checked kernels~\cite{witcert2026};
these establish request-conditioned sound bounds in serving loops. The
former's guarantees are deterministic only under stated preconditions,
are self-described as infeasible for batched serving, and --- by its
own limitations section --- its per-step bounds ``do not compose into
an end-to-end budget''; the latter budgets by union bound. Neither
touches routing, and neither survives adaptive decoding with a valid
request-level budget --- the two deficits this paper addresses.
\paragraph{PTQ with activation/second-order weighting.}
GPTQ/AWQ-style calibration~\cite{gptq2023,awq2024}, closed-form error
reconstruction~\cite{qera2025}, KL-Hessian rounding~\cite{yaqa2026}, and
information-theoretic allocation with provable rate
gaps~\cite{watersic2026} sharpen \emph{offline average} objectives; none
yields a per-instance serve-time certificate, a distinction WaterSIC's
own limitations section is candid about.
\paragraph{Dynamic-precision serving.} MorphServe~\cite{morphserve2025}
swaps quantized layers on load signals; DP-LLM~\cite{dpllm2025} selects
per-step precision by empirical thresholds; query-adaptive
QAQ~\cite{qaqqa2025} uses a trainable router over bit-planes. Request conditioning and serving integration appear ---
but guarantees do not, and the two properties never co-occur.
\paragraph{MoE quantization.} Expert-balanced
calibration~\cite{moequant2025}, global expert-level bit allocation with
router fine-tuning~\cite{gemq2026}, routing-consistency calibration
losses~\cite{vsraq2026,eaquant2025,eacmoe2025}, and
routing-consistency-regularized binarization (MoBiE~\cite{mobie2026},
withdrawn by its authors in April 2026): all offline;
GEMQ's measured $41.31\%$ expert-selection shift at 1.5-bit
demonstrates the fragility, none provides an online guarantee.
Production MXFP4 DeepSeek checkpoints quantize expert projections only
and exclude the router --- consistent with our gate-arm measurement.
The routing-replay and equal-byte comparisons that adjudicate what
these methods protect are the companion paper's subject; here they
enter only through the boundary \S\ref{sec:router} states.
\paragraph{Anytime-valid inference.} E-processes and confidence
sequences are mature for answer-quality monitoring of LLMs; what they
buy over a union budget is Ville's inequality,
$\Pr(\exists t\ge0: W_t\ge1/\delta)\le\delta$ for a nonnegative
supermartingale wealth process with $W_0=1$ --- the quantifier inside
the probability, no horizon declared, survival under adaptive stopping.
To our knowledge no prior work drives \emph{quantization precision
decisions} from such a budget; \S\ref{sec:av} is where we do, and also
where we show that wealth alone is \emph{not} a per-action risk
account.

\section{Limitations}
\label{sec:limits}
Four limits bound what this paper claims, and we state them at the level
that changes a reader's conclusion; the full enumeration, gate by gate,
is Appendix~\ref{app:limits}.
\emph{Choosing the certified predicate is outside the machinery, and
choosing it wrong is invisible from inside.} \S\ref{sec:router} is a
worked instance: the routing certificate issues, its soundness is a
theorem, its violation count is zero on nine models and
$485{,}138$ tokens, and every internal indicator is green --- while the
product is vacuous, because the predicate it establishes does not bound
the quantity anyone cares about (the routing-TV split identity (companion paper)). No amount of
tightening, and no check the framework can run on itself, would have
revealed that; only stepping outside and asking whether $P$ implies the
served effect does. We report this as the most transferable finding in
the paper and as a standing limitation of the approach, not as an
anecdote about one bad predicate.
\emph{The population is ours.} The conformal bound holds on a
$148$-candidate pool we constructed, not on runtime traffic and not along
the history stream an autoregressive request actually walks:
\begin{equation}
\underbrace{\Pr_{\text{pool}}\bigl(S_{\text{new}}>\hat q\bigr)\le0.0933}
      _{\text{proved, }M=148\text{ constructed histories}}
\;\;\not\Longrightarrow\;\;
\underbrace{\Pr_{\text{traffic}}\bigl(S_{\text{new}}>\hat q\bigr)\le0.0933}
      _{\text{what deployment needs}} .
\label{eq:populationgap}
\end{equation}
Closing the gap needs exchangeable sampling from real traffic, not a
sharper inequality.
\emph{The constants are measured, not bounded.} Propagation constants,
the query envelope $a_q$, and the gate threshold $w_{\mathrm{thr}}$ are
each measured on one run or one model; our own probe \emph{falsifies} the
first-order surrogate a tensorized proof would need, which is why
Theorem~\ref{thm:main} is stated for an arbitrary step perturbation.
\emph{The witness is local.} The physical account charges a band-norm of
the K/V residual, not served total variation; the bridge between them is
proved per read and open in the aggregate.

\section{Conclusion}
Certified precision for MoE serving must first certify the right
object --- and getting that wrong costs a misdirected budget, not a loose
bound (\S\ref{sec:consequences}). The companion paper's adjudication
shows it is not routing invariance ---
flips are pervasive and benign --- and a counterexample we aimed at
our own first design shows it is not e-process wealth dressed up as
per-action risk. What remains standing is precise, and we prove it:
quantization damage is a cumulative tail phenomenon of the served
output, and a per-step mass-weighted exponential-moment
budget bounds it --- a per-vocabulary bias-and-fluctuation bound weighted
by served probability, stated in served-logit units, with the map from
the KV-local witness still open. Theorem~\ref{thm:main} carries
that budget to a mass-weighted expected served-output total variation
and then, through a domination that needs only that expectation, to a
request-level tail bound:
\begin{equation}
\underbrace{\mathbb{E}\bigl[\mathrm{TV}_{\text{served}}\bigr]\le B}
      _{\text{mass-weighted expectation}}
\;\Longrightarrow\;
\underbrace{\Pr\bigl(\mathrm{TV}_{\text{served}}\ge\tau\bigr)\le B/\tau}
      _{\text{request-level tail}} .
\label{eq:twolevel}
\end{equation}
The composition is machine-checked, and ---
after our own propagation probe falsified the layer-local idealization
--- it is stated so that no such idealization enters it. The moment
hypothesis it consumes is itself proved from bounded differences alone,
with the bias carried explicitly, and the two are joined by a proved
equality of the conditional-expectation and product-measure forms. What
is left open is not a proof but an instantiation: supplying the
bounded-difference constants for a real network, in served-logit units,
from what the KV-side witness records. That is
the paper's methodological point as much as its technical one --- the
experiment that failed dictated the shape of the theorem that stands. The measured components are cheap (audit witnesses the write path
already produces); the controller's own serving
judgment costs $+3.8\%$ wall time (measured); what is not yet measured is
the graphs-on production fast path. The soundness is
machine-checked where it is probabilistic, and every gate that shaped
the system --- including the negatives that killed our own designs ---
is pre-registered and artifact-traceable.

\appendix
\section{Per-gate experimental record}
\label{app:exp}
\label{sec:exp}
All gates were pre-registered: verdict branches were committed to the
repository before the corresponding data existed, and every number
below is frozen in a machine-generated canon that traces to a run
artifact. Captures carry provenance stamps (machine, code hash, tree
digest) and declare their accounting keys; a conformance guard rejects
products whose declared keys are degenerate (a failure mode we hit and
formalized during this campaign).

\paragraph{Activation geometry (gate W2-a).}
Table~\ref{tab:w2a} summarizes per-stream spectra. Both models pass the
pre-registered concentration bar on attention streams; the MoE passes
on all three. Held-out energy is computed by fitting the top-$7.1\%$
eigenspace on even-parity requests and evaluating on odd-parity
requests (flat-spectrum baseline $0.071$).

\paragraph{KV-side honesty (gates K0, K0$'$, K0-R).}
The pre-registered kill-shots that redirected this paper: on synthetic
pools the certified-vacuous stratum was empty (43/43 layers
non-vacuous even at 4-bit masking) and no condemned read was rescuable
(rank correlation $0.872$ between realized distribution shift and
realized output error); the real-pool recheck is a
\emph{stronger negative} rather than a confirmation: across $189$
same-request pools (depth $\geq 64$, request-matched queries, RoPE
segment included) the INT6/4/3 family never produces a distribution
shift above the condemnation bar, so the condemnation stratum is empty
($0$ points) and that run's own pre-registered branch is
``inconclusive, no high shift'' --- there is nothing to rescue in the
deployed regime because nothing gets condemned. The shift-to-error rank
correlation there is $0.935$. The two pools fail the rescue hypothesis
in \emph{different} ways, which is why the second is the stronger
result:
\begin{equation}
\underbrace{1{,}528\ \text{condemned},\ 0\ \text{rescued}}
      _{\text{synthetic pool: rescue exists to test, and fails}}
\qquad
\underbrace{0\ \text{condemned}}
      _{\text{real pool ($189$): nothing to rescue}} .
\label{eq:twopools}
\end{equation}
The query-side spectrum pilot
($\rho_{90}=0.094$ in-sample against $0.538$ held-out energy at
$n\ll d$, verdict \emph{gray}) is retained as the small-sample
cautionary control that shaped the W2-a design.

\begin{table}[htbp]\centering\small
\begin{tabular}{lrrr}
\toprule
GLM-5.2, identical hardware & FP8 & W4AFP8 & ratio\\
\midrule
Aggregate throughput, $C{=}128$ (tok/s) & $1{,}789$ & $1{,}834$ & $1.03\times$\\
Aggregate throughput, $C{=}512$ (tok/s) & $1{,}827$ & $1{,}869$ & $1.02\times$\\
KV-cache capacity (resident tokens)     & $291{,}968$ & $1{,}087{,}808$ & $3.7\times$\\
Errors, full $1..128$ sweep             & $0$ & $0$ & ---\\
\bottomrule
\end{tabular}
\caption{Throughput is flat to within single-run spread at both
concurrencies --- we report the spread rather than claim equivalence,
since no restart confidence intervals were collected --- while
configured resident-token headroom grows $3.7\times$. On this
short-context workload neither arm reaches its KV wall.
Table~\ref{tab:convert} pressure-probes the headroom and reports what it
does and does not convert to. The capacity figure here is
configuration-derived, not an OOM limit.}
\label{tab:capacity}
\end{table}

\begin{table}[htbp]\centering\small
\setlength{\tabcolsep}{5pt}
\begin{tabular}{lcc}
\toprule
& no prefix reuse & session reuse\\
$26$k-token requests, GLM-5.2, W4AFP8 vs FP8 & (worst case) & (agent-shaped)\\
\midrule
Configured KV capacity            & \multicolumn{2}{c}{$3.73\times$}\\
$\rightarrow$ concurrent requests & $2.0\times$ & $2.24\times$\\
$\rightarrow$ warm-prefix hits    & --- (none)  & $1.93\times$\\
$\rightarrow$ \textbf{throughput} & $\mathbf{0.96}$--$\mathbf{1.01\times}$
                                  & $\mathbf{1.07}$--$\mathbf{1.27\times}$\\
$\rightarrow$ TTFT (p50)          & --- & $1.08$--$1.20\times$ ($0.92\times$ at $C{=}256$)\\
\bottomrule
\end{tabular}
\caption{What $3.7\times$ of KV capacity converts to, pressure-probed in
two workloads. \textbf{Left}: every request is a fresh slice of text, so
the prefix cache never hits (peak cached tokens $2{,}560$ and $0$ against
a $26{,}000$-token prompt). FP8 runs at $100\%$ pool occupancy and
W4AFP8 at $58\%$, so the capacity does buy $2\times$ the in-flight
requests --- and buys \emph{nothing} in throughput, because at this
prompt length the work is prefill-bound and the GPU is already
saturated. \textbf{Right}: each session keeps its own long context and
reuses it across turns, the shape a coding agent has. Now capacity
decides \emph{how many sessions stay cache-warm}: FP8 is pinned at
$10$--$11$ resident sessions per rank (its pool holds
$291{,}968/26{,}500=11.0$) at $100\%$ occupancy, while W4AFP8 holds
$23$--$24$ at $54\%$. Every link in the chain attenuates, and the last
one attenuates hardest --- what a hit saves is prefill, while steady-state
throughput is still set by decode and compute. Two caveats we state
rather than bury: the TTFT gain \emph{reverses} at $C{=}256$, where
W4AFP8's larger running batch leaves each request less compute, so the
throughput is bought with per-request latency; and each session completes
only $\approx1.9$ turns inside the measurement window, so $1.19\times$ is
a lower bound on what a real multi-turn session would show --- we do not
extrapolate it. Protocol: per-level prefix namespaces and a $90$\,s
warm-up per level; without both, prefixes carry across levels and the
measurement reports a spurious $1.58\times$.}
\label{tab:convert}
\end{table}

\paragraph{Capacity--cost evaluation (gate W3, measured).}
On GLM-5.2 served with the production stack (tp8/ep8/dp8, FP8 KV
cache, CUDA graphs, radix on; both arms at matched memory fraction
$0.85$ after a pre-registered arm death at $0.88$ taught us the W4AFP8
runtime workspace requirement), the W4AFP8 checkpoint is
throughput-neutral at matched concurrency --- $1{,}834$ vs $1{,}789$
aggregate tok/s at $C{=}128$, with TTFT/TPOT curves overlapping within
single-run spread and zero errors across the full $1..128$ sweep on both
arms (we report the spread rather than claim statistical equivalence,
since no restart confidence intervals were collected) ---
while its server-reported KV-cache capacity is $1{,}087{,}808$ vs
$291{,}968$ tokens (Table~\ref{tab:capacity}): $3.7\times$ the
configured resident-token headroom on identical hardware (a configuration-derived figure, not a
pressure-probed OOM limit; throughput parity is from single runs
without restart confidence intervals). An extended sweep to $C{=}512$
sharpens the interpretation: both arms hold a compute-saturation
plateau ($1{,}869$ vs $1{,}827$ tok/s at $C{=}512$) with zero errors
and matching queueing growth --- on this short-context workload neither arm reaches its KV
wall, so the capacity headroom is banked, not spent; it becomes
serving value exactly where resident tokens bind (long contexts, large
batch memory), which is the production regime the deployment reports
describe. Every level of both arms passes a machine adjudicator
(per-level liveness, matched level sets, capacity fields on record)
whose falsifiability is enforced by a real negative fixture: the run
in which one arm died mid-sweep and six dead windows masqueraded as
measurements.

\section{Request-conditioned activation geometry (precondition test)}
\label{app:geometry}
Both pillars consume a bound on how quantization error in a weight $W$
projects onto the activations $x$ the current request actually
produces: for error $E=W-\widehat W$, the served perturbation is $Ex$,
and $\mathbb E\|Ex\|^2 = \operatorname{tr}(E\,\Sigma_x\,E^\top)$ for
$\Sigma_x=\mathbb E[xx^\top]$. The decision-relevant question is
whether $\Sigma_x$ has usable low-dimensional structure, and whether
the structure measured on one phase of a request predicts the other.

\paragraph{Measurement.} We accumulate exact second moments
$\Sigma=X^\top X$ per (layer, stream, phase, request-parity) on the
three activation streams every transformer block exposes (attention
input, attention-output projection input, MLP/MoE input), on 16
teacher-forced requests (16 requests, 2048/4096-token prefills plus 128 forced decode steps; $3\times 10^4$-row order per stream, $n\gg d$). Split-half
(request-parity) held-out energy and prefill$\to$decode transfer are
controls against the small-sample bias that a pilot study on the KV
query side taught us to fear: at $n=96\ll d=448$ the in-sample
spectrum looked concentrated ($\rho_{90}=0.094$ of eigendirections) while
held-out energy at the matched rank reached only $0.538$ --- above the
flat-spectrum baseline but below the pre-registered $0.70$ bar, a
\emph{gray} verdict rather than a collapse. At $n\gg d$ the gap closes.

\begin{table}[h]\centering\small
\begin{tabular}{llccc}
\toprule
Model & Stream & $\rho_{90}$ & Held-out & Verdict\\
\midrule
Qwen2.5-7B & attn-in & $0.034$ & $0.829$ & concentrated\\
Qwen2.5-7B & o-proj-in & $0.0332$ & $0.742$ & concentrated\\
Qwen2.5-7B & mlp-in & $0.0444$ & $0.660$ & gray\\
DeepSeek-V2-Lite & attn-in & $0.0044$ & $0.946$ & concentrated\\
DeepSeek-V2-Lite & o-proj-in & $0.0264$ & $0.844$ & concentrated\\
DeepSeek-V2-Lite & moe-in & $0.0444$ & $0.792$ & concentrated\\
\bottomrule
\end{tabular}
\caption{Gate W2-a: activation second-moment spectra with held-out
validation. $\rho_{90}$ = fraction of eigendirections carrying $90\%$
of energy (prefill, per-layer median).}
\label{tab:w2a}
\end{table}

\paragraph{Findings (pre-registered thresholds).} On DeepSeek-V2-Lite
all three streams pass the concentrated verdict: the top
$0.44$--$4.4\%$ ($\rho_{90}$ per stream: $0.0044$, $0.0264$, $0.0444$) of eigendirections carry $90\%$ of energy,
and a subspace fitted on even-parity requests captures
$0.946$/$0.844$/$0.792$ of odd-parity energy at rank $7.1\%$ of
$d$ (flat-spectrum baseline: $7.1\%$). On Qwen2.5-7B the attention
streams pass ($\rho_{90}$ $0.034$/$0.0332$, held-out $0.829$/$0.742$); the MLP stream falls
in a pre-registered gray zone (held-out $0.660$ against a $0.70$ bar, $\rho_{90}$ $0.0444$). Crucially, coordinate-basis (diagonal) concentration is far
weaker (an order of magnitude more coordinates for the same energy): the structure is off-diagonal, so
per-channel scaling --- the diagonal special case --- cannot exploit
it. This replicates, on the weight side, a negative prior we measured
for diagonal query-weighting in KV quantization.

\begin{table}[htbp]\centering\small
\setlength{\tabcolsep}{3pt}\footnotesize
\begin{tabular}{@{}llll@{}}
\toprule
Arm & Hypothesis under test & Verdict & Key number\\
\midrule
Gate quantization & router survives INT4 & \textbf{killed} & $63.2\%$ flips\\
Upstream margins & exact gate suffices & \textbf{killed} & $93.7\%$ binding\\
Flip$\to$damage & flips predict damage & \textbf{dissociated} & RR $0.71$--$0.84$\\
Natural text & damage is average-case & \textbf{killed} & $\Delta$NLL $0.0002$\\
Value-transport & condemned reads rescuable & \textbf{killed} & $0$/$1{,}528$\\
\midrule
Served-output tail & the object that survives & \textbf{adopted} & TV $0.118$/bd.\ $0.159$\\
Propagation probe & first-order surrogate holds & \textbf{killed} & linearity $14.1$\\
Bounded differences & budget is non-vacuous & \textbf{model-dep.} & $0.0489$ (diagnostic)\\
Layer-local plug-in & product of per-layer moments & \textbf{model-dep.} & $0.0390$ (diagnostic)\\
Joint scalar moment & the functional the theorem needs & \textbf{supported} & pt.\ est.\ $0.0023$\\
Betting confidence seq. & bound free of sample extrema & \textbf{supported} & $\mathbb{E}[\mathrm{TV}]\le0.0322$\\
Deployed rounding & the surrogate is removable & \textbf{supported} & $\mathbb{E}[\mathrm{TV}]\le0.0552$\\
Cross-history spread & one history represents all & \textbf{killed} & $133{,}849\times$\\
Cross-history certif. & calibrated bound transfers & \textbf{supported} & $16/16$, $\le0.0254$\\
Lipschitz envelope & sound $c_\ell$ from weights & \textbf{killed} & $10^{221}$\\
Local envelope & contractive as a supremum & \textbf{killed} & $\sigma_{\max}\approx16$\\
\bottomrule
\end{tabular}
\caption{Every design hypothesis we tested, including the ones aimed at
our own proposals. The companion paper's five pre-registered arms
eliminate routing invariance as the certified object; the served-output
tail is what survives, and is what Theorem~\ref{thm:main} bounds. The last row is the measurement that
falsified the first-order propagation surrogate and thereby fixed the
form of the theorem. Numbers are the same artifacts cited in the text.}
\label{tab:adjudication}
\end{table}

\section{Proof sketch of Theorem 1 (machine-checked chain)}
\label{app:proof}

\noindent\emph{Proof (machine-checked). } Figure~\ref{fig:chain} shows
where this theorem sits in the chain and what it does not reach. The
statement above is
\texttt{request\_tail\_of\_served\_tv\_massweighted}; its Doob-supplied
form, taking bounded differences directly as input, is the single root
theorem \texttt{request\_tail\_of\_served\_tv\_doob\_massweighted}. Per draw, the
sub-Gaussian kernel gives $\mathrm{TV}\le\sqrt{\log Z(\omega)}$ via
$\mathrm{BC}\ge e^{-\sigma^2/4}$ and $\mathrm{TV}\le\sqrt{1-\mathrm{BC}^2}$;
concave Jensen for $\sqrt{\cdot}$ then $\log$, and a Fubini step giving
$\mathbb{E}_\omega Z=\sum_v p_h(v)\mathbb{E}_\omega[e^{\delta_h(\cdot,v)}]$,
yield the mass-weighted step bound --- note the weighting by $p_h(v)$ is
what keeps a large perturbation on a low-probability token from
dominating, and it is why the bias $b_h(v)$ can be priced rather than
assumed away. Convexity of $t\mapsto e^{\lambda t}$ on $[0,1]$ then gives the
Bernoulli domination
\begin{equation}\label{eq:bernoulliapp}
  \mathbb{E}\big[e^{\lambda\mathrm{TV}}\big]\;\le\;1+\mu\big(e^\lambda-1\big),
\end{equation}
whose mean-one e-process factor --- the right-hand side of
\eqref{eq:bernoulliapp} --- is exactly what the anytime ledger consumes; taking one decoding step's
whole draw vector as the ledger's event makes hypothesis (ii) a
\emph{conditional} statement at every history, which is what the ledger
needs. Axioms are the three standard ones, no \texttt{sorry}.

\section{The step budget is of the right order: caliber}
\label{app:rightorder}

\emph{Is the budget of the right order?} The theorem is only useful if
the quantity it consumes, $\sqrt{\log\sum_v p_h(v)B_h(v)}$, is far below
$1$ at a real operating point. We can bound the \emph{shape} of that
quantity without a new experiment, and we have to be exact about what
that does and does not show. Substituting one \emph{realized} draw
$\omega_0$ for $B_h(v)$ --- i.e.\ computing
$\sqrt{\log\sum_v p_h(v)e^{\delta_h(\omega_0,v)}}$ --- gives $0.230$
(median over the $8$ captured requests, $0.397$ at the maximum) against a
realized served TV of $0.118$. This is a \emph{realization} of the
object, not a bound on it: hypothesis~(ii) asks for
$\mathbb{E}_\omega[e^{\delta}]$, and a moment generating function is
precisely where a typical draw and its expectation part company --- a
rare large $\delta$ contributes nothing to $\omega_0$ and everything to
$\mathbb{E}_\omega$. So the honest reading is that the object the
theorem prices is of the right order where the system runs, and that
three things stand between that and a certificate: obtaining $B_h(v)$
\emph{a priori} from the witness rather than from a reference run;
replacing the realized draw by a bound on its $\omega$-expectation; and
replacing a median over eight requests by a bound at \emph{every}
history, which hypothesis~(iii) quantifies over. The next section's
propagation measurement attacks the first; \S\ref{sec:exp}
measures the third and finds it is where the real distance lies.

\section{The binomial route and its risk account}
\label{app:cprisk}

Counting those outcomes gives a Clopper--Pearson bound of $0.2501$ on the
rate at which a further history \emph{fails to certify}. That is a bound
on the certification procedure rather than on the truth, and it is the
weakest instrument available for the job; the next paragraph replaces it.
A complete risk account reads
\begin{equation}\label{eq:risk}
  \delta_{\mathrm{total}}\;=\;\delta_{\mathrm{ledger}}
  +\alpha_{\mathrm{draw}}^{\mathrm{total}}
  +p_{\mathrm{history}}^{\mathrm{UCB}}
  +\beta_{\mathrm{history}},
\end{equation}
the last term pricing the confidence in the history-failure-rate estimate
itself. Reaching a $1\%$
history-failure rate at $99\%$ confidence with zero violations needs on
the order of $459$ histories by this route; a hundred would reach roughly
$4.5\%$.

\section{Adjudication notes for the dual-accounted admission table}
\label{app:ledgernotes}
Three calibers govern how Table~\ref{tab:ledger} may be read.
\emph{Cross-run comparison.} A separate offline dual-accounting run over
$36$ request accounts puts the median \emph{upper bound} on the union
fallback rate at $31.0\%$; different run, different population, and
neither bounds the other. The coverage gain is also not like-for-like:
the guaranteed object changes from a per-event family to a cumulative
account, which the run's own caliber note states.
\emph{Physical-gate non-vacuity.} The physically-accounted gate
($\texttt{WITCERT\_CUMLOSS\_MODE}$$=$\texttt{phys}) stays non-vacuous ---
the tail term is $0.67\%$ of the accumulated mean term, not of the
budget --- with the certified bound sitting $4.6\%$ above the realized
witness at the request endpoint (median over $16$ endpoints).
\emph{Judgment cost.} The cost is not monotone in the budget: the tight
arm runs $1.5\%$ \emph{faster} than the union baseline while the open
arm runs $3.8\%$ slower, so the $+3.8\%$ quoted elsewhere is that arm's
figure, not a property of the mechanism.

\section{The history cost of the certificate}
\label{app:histcost}

Continuing from the marginal/conditional distinction of \S\ref{sec:quantifier}: What the order statistic changes is the
cost of eventually getting there. Both terms must be
priced together, since $1/(N+1)+\alpha\ge\alpha$. Splitting a $1\%$
\emph{history} sub-budget evenly gives $\alpha=0.005$ and $N=199$
histories \emph{at the marginal caliber}; splitting it unevenly is
cheaper still ($\alpha=0.001$ needs $N=111$). A training-conditional
statement at the same target needs more, and how much more depends on
the population size. We stress that this is the history term of \eqref{eq:risk}
alone. A request-level budget must satisfy
\begin{equation}\label{eq:budget}
  \delta_{\mathrm{ledger}}+\alpha_{\mathrm{conf}}+\tfrac{1}{N+1}
  \;\le\;\delta_{\mathrm{req}},
\end{equation}
so by \eqref{eq:budget} the sample size follows from
$\delta_{\mathrm{req}}$ only after $\delta_{\mathrm{ledger}}$ is frozen;
at our present configuration the two accounts together are
$1\%+2.54\%$, not $1\%$. Comparing
against zero-event Clopper--Pearson at the same total is
budget-dependent in the same way: $459$ histories if its $1\%$ is
charged entirely to the bound value, $1{,}057$ if its confidence is
charged separately as we charge ours. We quote the conservative
end --- $2.3\times$ fewer histories --- and note the honest range is
$2.3$ to $5.3\times$. Figure~\ref{fig:histcost} traces the requirement
across targets.

\begin{figure}[htbp]\centering
\includegraphics[width=0.58\linewidth]{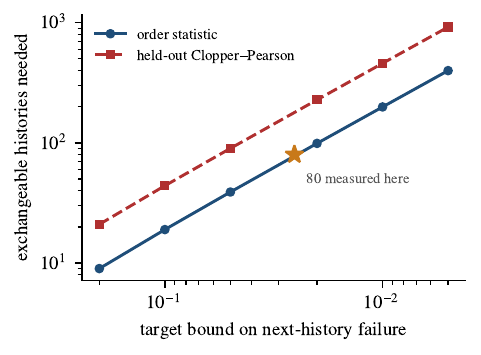}
\caption{What the instrument costs in histories, at the \emph{marginal}
caliber and in the large-population limit. Number of exchangeably sampled
histories needed to certify a given bound on next-history failure,
splitting the budget evenly between the two terms; the Clopper--Pearson
curve is the zero-event requirement at $99\%$ confidence. A
training-conditional guarantee costs more, and a finite pool costs more
still --- our own $80$-history run reaches $0.0254$ marginally but
$0.0933$ conditionally. The two follow different scaling laws, so the gap widens as
the target tightens --- which is the long-run argument for the order
statistic, independent of the factor we measure at $80$ histories
($\star$).}
\label{fig:histcost}
\end{figure}

\section{Full limitations, gate by gate}
\label{app:limits}
(i) The routing arms run on DeepSeek-V2-Lite, a $64$-expert model with
greedy top-$k$ selection. We have since re-measured on the serving-line
model itself --- DeepSeek-V4-Flash, $256$ experts top-$6$ under
bias-corrected routing --- and the verdict is unchanged and stronger
($99.0\%$ binding at INT4 against $93.7\%$ on the proxy, zero soundness
violations). That gap is now closed the rest of the way: the
nine-model sweep of \S\ref{sec:router} includes GLM-5.2 ($99.9\%$),
MiniMax-M2.7, Qwen3, gpt-oss, Kimi-K2.5, Mixtral and Llama-4-Scout, so
the residual exposure is no longer \emph{which family} but \emph{which
selection rule}: a two-stage (grouped) rule must charge both stages, and
only the second is the familiar one,
\begin{equation}
\underbrace{m_{\mathrm{group}}>4L\varepsilon_\infty}_{\text{group stage
(a sum of two perturbed scores)}}
\quad\wedge\quad
\underbrace{m_{\mathrm{expert}}>2L\varepsilon_\infty}_{\text{within-group
stage}} .
\label{eq:groupedcert}
\end{equation}
Rules we have not measured --- learned or hashed assignment, or
$>2$-stage selection --- are not covered by
(\ref{eq:groupedcert}) and remain open. (ii) The upstream logit-error radius used online must itself be
certified from activation geometry (Appendix~\ref{app:geometry}); the
teacher-forced measurement in \S\ref{sec:router} bounds achievable
behavior, not the deployed bound's tightness. (iii) Our first
router-margin arm conflated perturbation sources (gate vs.\ upstream);
we report both arms and the correction openly. (iv) The value-transport
rescue hypothesis for KV quantization died its pre-registered death on
synthetic pools (rank correlation between attention-distribution
shift and realized output error $0.872$; zero rescuable reads among
1,528 condemned). The real-pool recheck is not a confirmation of the
death but a stronger negative: under the production quantizer family no
read even reaches the condemnation bar, so the condemnation stratum is
empty ($0$ points) and that run's own pre-registered
verdict branch is ``inconclusive, no high shift''; the rank correlation there is
$0.935$. This paper claims nothing about
output-visible rescue. (v) Cross-layer propagation of weight-error
into router inputs is measured, not bounded; hidden-state bridges
remain empirical tier. (vi) The armed controller's current action is
KV-write precision; expert-weight precision switching is motivated and
priced here but not implemented. (vii) The witness-to-served bridge is
closed as a \emph{theorem} (Theorem~\ref{thm:main}) but not as a
\emph{deployment}: its per-layer budget $C_\ell$ is stated in
served-logit units, whereas the runtime witness accounts in KV-local
units, and the constant relating them is exactly what our propagation
probe failed to certify (the response saturates, so the first-order
surrogate that would supply it is invalid). Until that constant is
bounded --- not estimated --- the theorem prices a budget the deployed
ledger does not yet emit. (viii) The bounded-difference constants
$c_\ell$ are now \emph{measured} directly, and the budget they produce is
non-vacuous ($0.0489$, with the bias measured rather than
assumed), but they are empirical maxima --- lower bounds on the supremum
a sound certificate needs. Closing this means a genuine envelope for the
compressed path; the composition rule is proved
(\texttt{bdd\_diff\_of\_lipschitz}), but instantiating it from the
deployed weights yields $10^{221}$ and rules a \emph{global} envelope
out. The local route is formalized with its extra obligation explicit
(\texttt{bdd\_diff\_of\_lipschitz\_on}, \texttt{lip\_iterComp\_on}); the
ladder it buys is large and still not enough:
\begin{equation}
\underbrace{10^{221}}_{\text{global envelope}}
\;\longrightarrow\;
\underbrace{\sigma_{\max}\approx16.0}_{\text{local, along the trajectory}}
\;\longrightarrow\;
\underbrace{\text{effective rank }42}_{\text{what actually settles it}} ,
\label{eq:lipladder}
\end{equation}
a hundred-plus orders of magnitude for the first step, yet
$\sigma_{\max}>1$ is expansive, so a Hoeffding bound is still vacuous. Our own reading
of the typical ratio ($0.0638$) as a Lipschitz constant was wrong and we
retract it. The repair the rank indicates is a Bernstein-type moment
bound calibrated to it rather than a Hoeffding bound calibrated to a
supremum. Deriving and machine-checking that bound is the concrete next
step, together with verifying forward invariance online. (ix) Our per-layer plug-in numbers ($0.0390$,
$0.0489$) are diagnostics under the layer-local additive model, not
instantiated budgets: multiplying single-layer moments assumes exactly
the additivity the saturation measurement rejects. One interface
mismatch remains in the same direction --- the measured quantity is a
maximum against a single baseline, not a diameter over pairs of draws
(the other, a vocabulary-shared constant, we removed by indexing per
token in the theorem). The experiment that removes both at once, a joint
resample of every layer reduced to one scalar per draw, we report. An empirical Bernstein bound gives
$0.2086$, but it takes its range $R$ from the \emph{sample} maximum
rather than a deterministic bound, so it is a diagnostic, not sound. It mirrors the Bernstein row of
Table~\ref{tab:ladder}: there the range was too \emph{large} and the
bound went vacuous; here it is too small and the bound stops being
valid. The gap is closed properly instead:
total variation is deterministically bounded in $[0,1]$, so a
betting-style confidence sequence certifies
$\mathbb{E}_\omega[\mathrm{TV}]\le0.0322$ with no appeal to a sample
extremum. The proxy then goes too: repeating an identical prompt draws
fresh deployed rounding each time, and on those real draws --- compression
verified to occur, every realized total variation distinct --- the same
sequence gives $\mathbb{E}_\omega[\mathrm{TV}]\le0.0552$:
\begin{equation}
\underbrace{\mathbb{E}_\omega[\mathrm{TV}]\le0.0322}
      _{\text{synthetic draws, no sample extremum}}
\qquad
\underbrace{\mathbb{E}_\omega[\mathrm{TV}]\le0.0552}
      _{\text{deployed rounding, }160\text{ distinct real draws}} .
\label{eq:csboth}
\end{equation}
What separates this paper from a deployed
certificate is therefore neither the surrogate nor the inequality but
the \emph{quantifier}: the theorem asks its hypothesis at every history
a request visits. We establish it at $80$ histories drawn
exchangeably from a stated pool spanning $133{,}849\times$ in
sensitivity, and certify transfer to an unseen member with failure
probability $\le0.0254$; all $16$ held-out histories certify. \enlargethispage{6\baselineskip}%
Two distinctions carried that result, and neither is pedantry.
Reading ``did not certify'' as ``violated'' is the same error, one level
up, as reading a sample mean as a bound --- we made both before catching
them. And an unmet exchangeability premise turns a bound into an
efficiency ceiling, not a risk account, which is why we rebuilt the
history family as a without-replacement sample. The population is still
ours, not runtime traffic; risk is reported as in \eqref{eq:risk}.

\bibliographystyle{plain}
\bibliography{refs}

\end{document}